\documentclass[acmsmall]{acmart}
\AtBeginDocument{%
  \providecommand\BibTeX{{%
    \normalfont B\kern-0.5em{\scshape i\kern-0.25em b}\kern-0.8em\TeX}}}

\usepackage{amsmath, amsfonts}
\usepackage{multirow, rotating, wasysym, url}
\usepackage[ruled,linesnumbered,vlined]{algorithm2e}
\usepackage{hyperref}
\usepackage{lipsum}
\usepackage{algorithmic}
\usepackage{graphicx}
\usepackage{balance}
\usepackage{textcomp}
\usepackage{xcolor}
\usepackage{subfigure}
\usepackage{enumitem}
\usepackage{array}
\usepackage{makecell}
\usepackage{soul}
\usepackage{multicol}
\usepackage{framed}

\newcommand{\ie}{\textit{i.e.}}
\newcommand{\bbb}{\noindent\textbf}

\newcommand{\presec}{\vspace{-0.0cm}}
\newcommand{\postsec}{\vspace{-0.0cm}}

\newcommand{\algo}{HotSketch\xspace}
\newcommand{\name}{CAFE\xspace}

\begin{document}

\setcopyright{acmlicensed}
\acmJournal{PACMMOD}
\acmYear{2024} \acmVolume{2} \acmNumber{N1 (SIGMOD)} \acmArticle{51} \acmMonth{2}\acmDOI{10.1145/3639306}

\title{\name: Towards Compact, Adaptive, and Fast Embedding for Large-scale Recommendation Models}

\author{Hailin Zhang}
\authornote{Both authors contributed equally to this research.}
\email{z.hl@pku.edu.cn}
\orcid{0009-0000-4188-7742}
\affiliation{%
  \department{School of Computer Science \& Key Lab of High Confidence Software Technologies}
  \institution{Peking University}
  \state{Beijing}
  \country{China}
}

\author{Zirui Liu}
\authornotemark[1]
\email{zirui.liu@pku.edu.cn}
\orcid{0000-0001-9062-6565}
\affiliation{%
  \department{School of Computer Science \& Key Lab of High Confidence Software Technologies}
  \institution{Peking University}
  \state{Beijing}
  \country{China}
}

\author{Boxuan Chen}
\email{2100012923@stu.pku.edu.cn}
\orcid{0009-0006-3719-2685}
\affiliation{%
  \department{School of Electronics Engineering and Computer Science}
  \institution{Peking University}
  \state{Beijing}
  \country{China}
}

\author{Yikai Zhao}
\email{zyk@pku.edu.cn}
\orcid{0000-0003-2495-7774}
\affiliation{%
  \department{School of Computer Science \& Key Lab of High Confidence Software Technologies}
  \institution{Peking University}
  \state{Beijing}
  \country{China}
}

\author{Tong Zhao}
\email{zhaotong@pku.edu.cn}
\orcid{0009-0005-7201-2152}
\affiliation{%
  \department{School of Computer Science}
  \institution{Peking University}
  \state{Beijing}
  \country{China}
}

\author{Tong Yang}
\authornote{Bin Cui and Tong Yang are the corresponding authors.}
\email{yangtongemail@gmail.com}
\orcid{0000-0003-2402-5854}
\affiliation{%
  \department{School of Computer Science \& Key Lab of High Confidence Software Technologies}
  \institution{Peking University}
  \state{Beijing}
  \country{China}
}

\author{Bin Cui}
\authornotemark[2]
\email{bin.cui@pku.edu.cn}
\orcid{0000-0003-1681-4677}
\affiliation{%
  \department{School of Computer Science \& Key Lab of High Confidence Software Technologies \& Institute of Computational Social Science, Peking University (Qingdao)}
  \institution{Peking University}
  \state{Beijing}
  \country{China}
}

\renewcommand{\shortauthors}{Hailin Zhang, et al.}

\begin{abstract}

Recently, the growing memory demands of embedding tables in Deep Learning Recommendation Models (DLRMs) pose great challenges for model training and deployment.
Existing embedding compression solutions cannot simultaneously meet three key design requirements: memory efficiency, low latency, and adaptability to dynamic data distribution.
This paper presents \underline{\textbf{\name}}, a \underline{\textbf{C}}ompact, \underline{\textbf{A}}daptive, and \underline{\textbf{F}}ast \underline{\textbf{E}}mbedding compression framework that addresses the above requirements. 
The design philosophy of \name is to dynamically allocate more memory resources to important features (called hot features), and allocate less memory to unimportant ones.
In \name, we propose a fast and lightweight sketch data structure, named \algo, to capture feature importance and report hot features in real time.  
For each reported hot feature, we assign it a unique embedding. 
For the non-hot features, we allow multiple features to share one embedding by using hash embedding technique. 
Guided by our design philosophy, we further propose a multi-level hash embedding framework to optimize the embedding tables of non-hot features.  
We theoretically analyze the accuracy of \algo, and analyze the model convergence against deviation. 
Extensive experiments show that \name significantly outperforms existing embedding compression methods, yielding $3.92\%$ and $3.68\%$ superior testing AUC on Criteo Kaggle dataset and CriteoTB dataset at a compression ratio of $10000\times$.
The source codes of \name are available at GitHub~\cite{cafecode}.

\end{abstract}



\received{July 2023}
\received[revised]{October 2023}
\received[accepted]{November 2023}

\maketitle

\presec
\section{introduction}
\postsec

\subsection{Background and Motivation}

In recent years, embedding techniques are widely applied in various fields in database community, such as cardinality estimation~\cite{DBLP:journals/pvldb/LiuD0Z21,DBLP:journals/pvldb/KwonJS22}, query optimization~\cite{DBLP:journals/pvldb/ZhaoCSM22,DBLP:journals/pvldb/ChenGCT23}, language understanding~\cite{DBLP:journals/pvldb/KimSHL20}, entity resolution~\cite{DBLP:journals/pvldb/EbraheemTJOT18,DBLP:journals/jcst/DingGLLZLPD23}, document retrieval~\cite{DBLP:journals/pvldb/HuangSLPKY20}, graph learning~\cite{DBLP:journals/pvldb/KochsiekG21,DBLP:journals/pvldb/YangSX0LB20}, and advertising recommendation~\cite{DBLP:conf/sigmod/MiaoSZZNY022}, to learn the semantic representations of categorical features.
Among these fields, Deep Learning Recommendation Models (DLRMs) are one of the most important applications of embedding techniques: they account for 35\% of Amazon's revenue in 2018~\cite{chui2018notes,underwood2019use,xie2018personalized}, and consume more than 50\% training and 80\% inference cycles at Meta’s data centers in 2020~\cite{DBLP:conf/hpca/GuptaWWNR0CHHJL20,DBLP:journals/corr/abs-2003-09518}.

As shown in Figure~\ref{fig:struc:dlrm}, a typical DLRM vectorizes categorical features into learnable embeddings, and then feeds these embeddings into downstream neural networks along with other numerical features~\cite{DBLP:conf/recsys/Cheng0HSCAACCIA16,DBLP:journals/corr/abs-1906-00091,DBLP:journals/chinaf/ShaoWCZW22,DBLP:journals/dase/YuanZYHX23,DBLP:journals/dase/MengHZWZ23,DBLP:journals/jcst/HanWN23}. 
Recently, with the exponential increase of categorical features in DLRM, the memory requirements of embedding tables have also skyrocketed, which creates formidable storage challenges in various applications~\cite{DBLP:conf/mlsys/YinAWL21,DBLP:conf/isca/MudigereHHJT0LO22}.
Therefore, it is highly desired to devise a framework that can effectively compress the embedding tables into limited storage space without compromising model accuracy. 
In this paper, we focus on compressing the embedding tables of extremely large-scale DLRMs.

\begin{figure}[htbp]
    \centering  
    \includegraphics[width=0.7\linewidth]{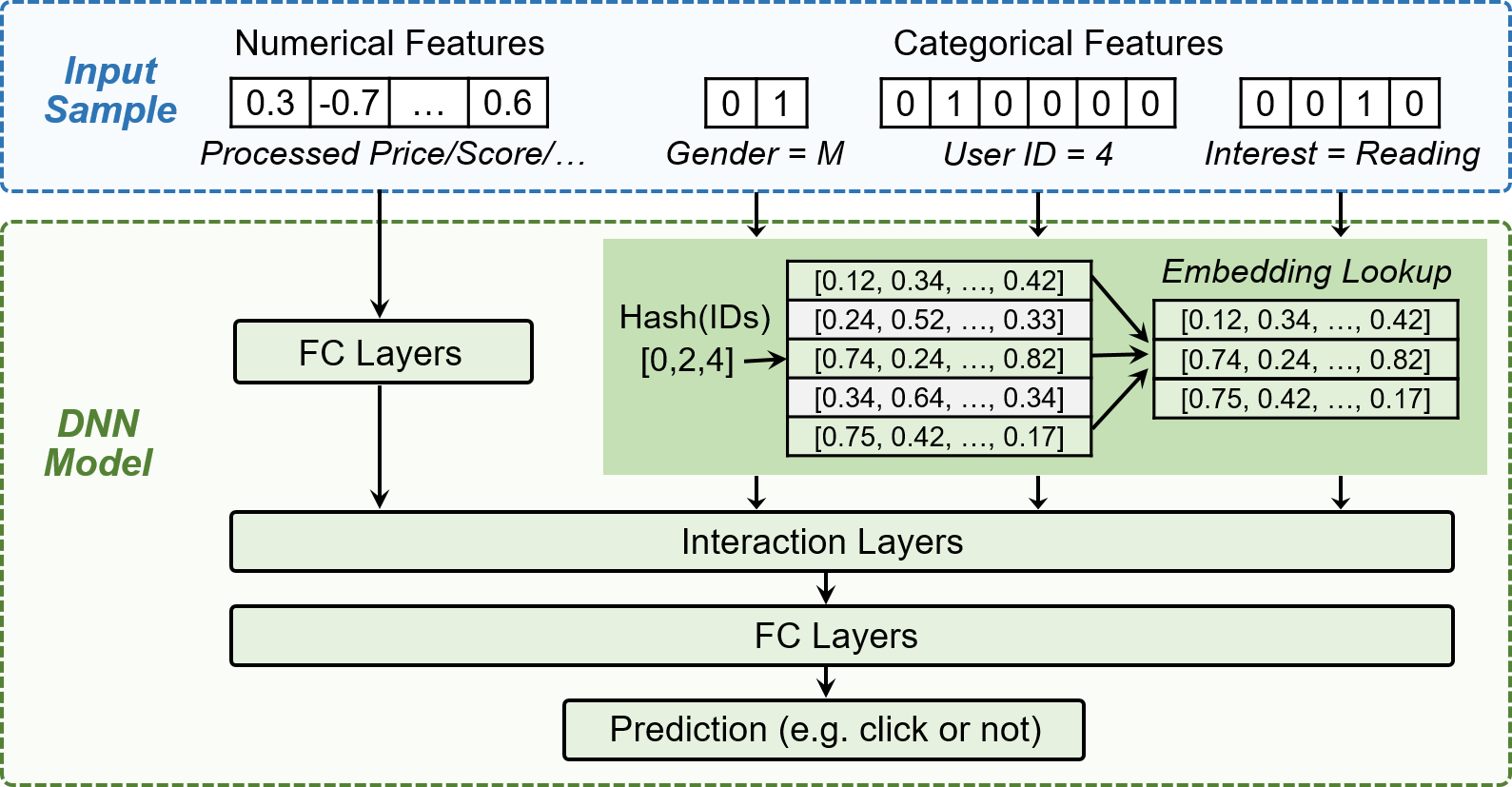}
    \caption{Overview of DLRM.}
    \label{fig:struc:dlrm}
\end{figure}

DLRM has two training paradigms: offline training and online training. 
(1) In offline training, the training data is collected in advance, and the model is deployed for use after the entire training process.
(2) In online training, the training data is generated in real time, and the model simultaneously updates parameters and serves requests.
This paper mainly focuses on the scenario of online training as it is more difficult.
Generally, compression methods for online training can be directly applied to offline training.
Embedding compression in online training has three important design requirements, which are as follows:

\begin{itemize}[leftmargin=*,parsep=0pt,itemsep=0pt,topsep=2pt,partopsep=2pt]
\item \textbf{Memory efficiency.} 
For extremely large-scale DLRMs, it is challenging to maintain model quality within memory constraints.
While distributed instances can help manage large-scale embedding tables, they come with a significant communication overhead~\cite{DBLP:journals/pvldb/MiaoZSNYTC21,DBLP:conf/recsys/WangWLLYLLAGDSL22}.
Furthermore, training and deployment of embedding tables often occur on edge or end devices with smaller storage capacities, making the memory issue even worse~\cite{DBLP:conf/mlsys/PansareKACSTV22}.
On the other hand, since model quality directly impacts profits, even a small change of 0.001 in DLRM's AUC (area under the ROC curve) is considerable~\cite{DBLP:conf/ijcai/GuoTYLH17}.
Existing compression methods often lead to severe model degradation when memory constraints are small~\cite{DBLP:conf/mlsys/ZhaoXJQDS020}, emphasizing the need for memory-efficient compression methods that maintain model quality.

\item \textbf{Low latency.}
Low latency is a vital requirement in practical applications, as latency is a key metric of service quality~\cite{DBLP:conf/isca/GuptaHSWRWL0W20}.
Embedding compression methods must be fast enough not to introduce significant latency.

\item \textbf{Adaptability to dynamic data distribution.}
In online training, the data distribution is not fixed as in offline training.
We calculate the KL divergence (an asymmetric measure of the distance between distributions) between the feature distributions on each day within three common public datasets, and plot the heatmaps in Figure~\ref{fig:intro:kl}.
In each heatmap, the block in row $i$, column $j$ shows the KL divergence between the distributions on day $i$ and day $j$.
There is a significant difference between the feature distributions, and generally the greater the number of days between, the greater the difference.
Existing advanced compression methods often exploit feature frequencies explicitly~\cite{DBLP:conf/recsys/ZhangLXKTGMDPIU20,DBLP:conf/isit/GinartNMYZ21} or capture feature importance implicitly~\cite{DBLP:conf/www/ZhaoLLTGSWGL21,kong2022autosrh}, which are inspired by the observation that feature popularity distributions are highly skewed, fitting zipfian~\cite{DBLP:conf/mlsys/YinAWL21} or powerlaw distributions~\cite{DBLP:conf/sigir/ZhangC15}.
However, most of them rely on fixed data distributions and cannot be applied to dynamic data distributions, demanding new adaptive compression method suitable for online training.

\end{itemize}

\begin{figure}[h]
\centering
\includegraphics[width=0.7\linewidth]{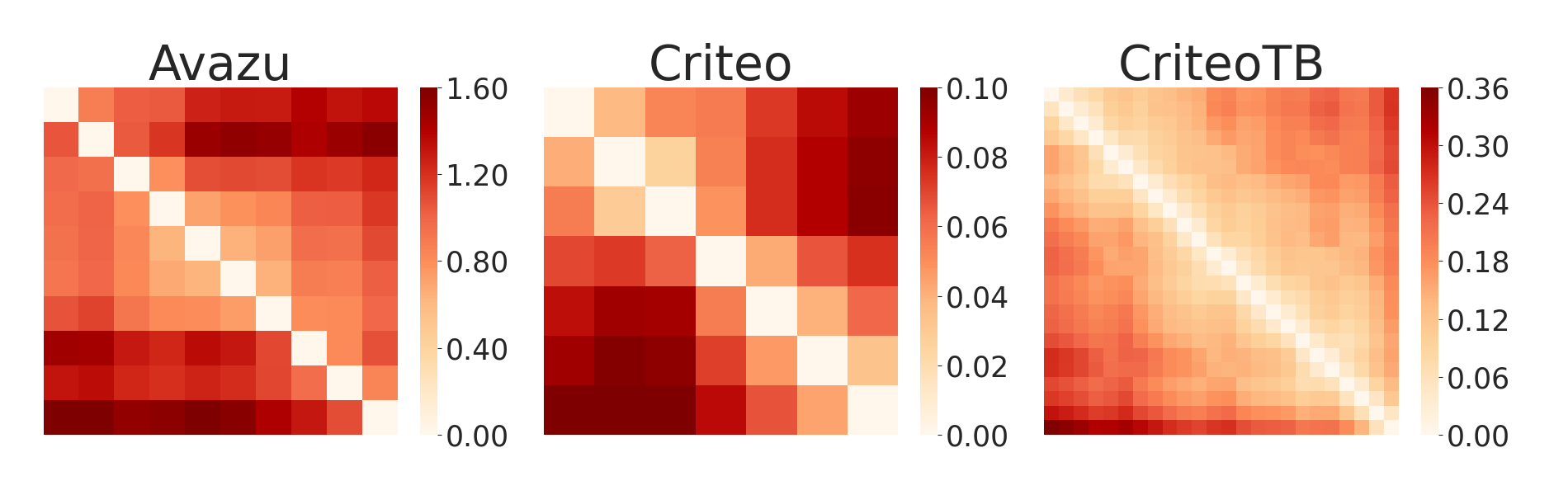}
\caption{KL divergence between distributions on each day.}
\label{fig:intro:kl}
\end{figure}

\subsection{Limitations of Prior Art}\label{sec:intro:limitation}

Existing embedding compression methods can be generally categorized into two types: row compression and column compression. 
As column compression primarily aims at enhancing model quality rather than compressing to a specific memory limit, our focus is on row compression, including hash-based and adaptive methods.

\bbb{Hash-based methods.}
These methods utilize simple hash functions to map features into embeddings with collisions~\cite{DBLP:conf/icml/WeinbergerDLSA09,DBLP:conf/kdd/ShiMNY20,DBLP:conf/cikm/YanWLLLXZ21}.
They restrict the number of embedding vectors to fit within the memory budget, causing different features to potentially share an embedding vector when a hash collision occurs.
Despite their simplicity and convenience, which have resulted in widespread industry use, these methods are not very memory-efficient.
Pre-determined hash functions distort the semantic information of features, often leading to a substantial decline in model accuracy.
For each feature, the gradients of other hash-collided features will be updated to the same embedding, resulting in deviations from the original convergence direction.
Integrating feature frequency information~\cite{DBLP:conf/recsys/ZhangLXKTGMDPIU20} can enhance hash-based methods' model quality in offline training but cannot be applied to online training.

\bbb{Adaptive methods.}
To accommodate online training, adaptive methods distinguish and track important features throughout the training process.
AdaEmbed~\cite{lai2023adaembed} logs the importance scores of all features, dynamically reallocates embedding vectors for critical features, and discards embeddings of less important features.
While it can adapt to data distribution, its compression ratio is constrained by the storage of importance scores, which increases linearly with the total number of features.
Thus, it cannot compress embedding tables to a small memory budget and still needs distributed training for large models, resulting in low memory efficiency.
It also needs to sample and check data to determine whether to migrate embeddings, which can increase overall latency.

In summary, existing methods fail to meet all three critical requirements for DLRM: memory efficiency, low latency, and adaptability.
In this paper, we aim to propose an embedding compression method that is memory-efficient, adaptive, and ensures low latency.

\subsection{Our Proposed Method}

We introduce \underline{\textbf{\name}}, a \underline{\textbf{C}}ompact, \underline{\textbf{A}}daptive, and \underline{\textbf{F}}ast \underline{\textbf{E}}mbedding compression method, which, to our knowledge, is the first to satisfy all three design requirements.
\textbf{(a) Memory efficiency:}
\name allocates unique embedding vectors to important features and shared embedding vectors to less important features, thereby preserving model quality.
A light-weight sketch, \algo, distinguishes these features, with its memory usage being linear to the number of important features, enabling high compression ratios.
Consequently, \name manages to maintain good model quality within tight memory constraints.
\textbf{(b) Low latency:}
\name entails only several hash processes and potentially one additional embedding lookup, incurring negligible time overhead beyond the standard embedding layers and thus maintaining low latency during serving.
\textbf{(c) Adaptability to dynamic data distribution:}
\name incorporates an embedding migration process that takes effect when a feature's importance score changes, ensuring that vital features are always identified even when data distribution changes during online training.
On Criteo dataset, compared to existing methods, \name improves the model AUC by $1.79\%$ and reduces the training loss by $2.31\%$ on average.

To achieve a high compression ratio without compromising model quality, we utilize a sketch structure to distinguish and record important features from a highly skewed Zipf distribution.
Sketches are a class of probabilistic data structures for processing large-scale data streams, and are naturally suitable for handling streamed features in online training.
Specifically, we extend SpaceSaving Sketch~\cite{DBLP:conf/sigmod/Ting18}, an advanced sketch algorithm with small error, to design \algo, a less memory-consuming structure to store important DLRM features with a theoretically guaranteed error bound.
Being a light-weight data structure, \algo incurs negligible time overhead, facilitating fast training and inference.
Since \algo's memory usage is only linear to the number of important features, \name can compress to any given memory constraints.
With \algo, we allocate unique embeddings to a handful of important features and shared embeddings to a vast majority of long-tail features, achieving memory efficiency.

To adapt \name to online training, where important features can change dynamically, we enable features to migrate between unique and shared embedding tables.
If a feature's importance score exceeds a relative threshold in \algo, it is deemed important and allocated a unique embedding.
Conversely, if a feature's importance score drops below a relative threshold, its unique embedding migrates to the shared embedding table.

To further optimize \name, we divide features into more groups by importance scores.
While the most critical features are still allocated unique embeddings, other features are assigned a varying number of hashed embedding vectors.
This multi-level design further improves the model AUC by $0.08\%$ on Criteo dataset.

\subsection{Main Contribution}

\begin{itemize}[leftmargin=*]
    \item We introduce \name, a compact, adaptive, and fast embedding compression method.
    \item We propose \algo, a light-weight sketch structure to discern and record features' importance scores.
    \item We provide a theoretical analysis of \algo's effectiveness, and elucidate how \name's design contributes to the convergence of compressed DLRMs.
    \item We evaluate \name on representative DLRM datasets, achieving $3.92\%$, $3.68\%$, $5.16\%$ higher testing AUC and $4.61\%$, $3.24\%$, $11.21\%$ lower training loss at $10000\times$ compression ratio compared to existing method.
\end{itemize}

\presec
\section{Preliminary}\label{sec:prelim}
\postsec

In this section, we elaborate on the architecture of DLRMs in Section~\ref{sec:2.1} and provide a formal definition of the embedding compression problem in Section~\ref{sec:2.2}.

\subsection{DLRM}\label{sec:2.1}
Figure~\ref{fig:struc:dlrm} illustrates the overall architecture of DLRM.
Each dataset of DLRM has several categorical feature fields and numerical feature fields. For example, in Figure~\ref{fig:struc:dlrm}, gender, user ID and interest are categorical fields, while price and score are numerical fields. Each field has a certain number or a certain range of possible values, called features.
Categorical and numerical features are transformed into representations using embedding vectors and fully-connected layers, respectively.
The representations are then fed into interaction layers and fully-connected layers for final predictions.
The prediction may be a category for classification tasks such as click-through-rate and conversion-rate prediction, or a score for regression tasks such as score prediction.
There are many variants of DLRM, such as WDL~\cite{DBLP:conf/recsys/Cheng0HSCAACCIA16}, DCN~\cite{DBLP:conf/kdd/WangFFW17}, DIN~\cite{DBLP:conf/kdd/ZhouZSFZMYJLG18}; while they all utilize the same embedding layer, they explore different forms of interaction layers and neural network layers to enhance model performance.

The size of a DLRM does not depend on the model structure, but on the number of unique categorical features in the dataset. The model parameters of DLRMs can be divided into two parts: the embedding table and the neural network. The former contains embeddings for all categorical features, \ie{}, one embedding per feature if uncompressed. The latter is a network that interacts these embeddings and outputs predictions. The number of parameters in the embedding table depends on the dataset: if there are $n$ unique categorical features in the dataset, and the dimension of embeddings is $d$, then the number of parameters is $n\times d$. In DLRMs, the size of the neural network part (just a few layers of matrix multiplication) is negligible compared to large embedding tables. Based on previous research works~\cite{jiang2019xdl,DBLP:conf/cikm/ZhaoZXQJ019,DBLP:conf/sc/XieRLYXWLAXS20,DBLP:conf/isca/MudigereHHJT0LO22,DBLP:conf/kdd/LianYZWHWSLLDLL22}, we consider DLRMs with more than 100 million parameters as large-scale, and DLRMs with more than 10 billion parameters as extremely large-scale.

In DLRMs, categorical features are viewed as one-hot vectors where only the $i$-th position is set to 1 and the rest are set to 0, facilitating the retrieval of the corresponding row vector from the embedding table.
Each input data, sampled from distribution $\mathcal{D}$, contains categorical features $x_{cat}$, numerical features $x_{num}$, and a label $y$.
We denote $E$ as the embedding tables and $f$ as the other neural network layers, then the process of minimizing the loss can be formulated as follows:

\begin{equation}
\min\limits_{E,f} \mathbb{E}_{(x_{cat}, x_{num},y)\sim \mathcal{D}}\mathcal{L}(y,f(E(x_{cat}),x_{num})).
\end{equation}

In each iteration, after the forward pass, an optimizer such as Adam~\cite{DBLP:journals/corr/KingmaB14} is applied to update the embedding table and other parameters. 
Frequently used notations in this paper are detailed in Table~\ref{tab:note}.

    

\begin{table}[!ht]
\footnotesize
    \caption{Symbols frequently used in this paper.}\label{tab:note}
    \begin{tabular}{|c|c|c|c|}
    \toprule[0.5pt]
    \textbf{Symbol} & \textbf{Meaning} & \textbf{Symbol} & \textbf{Meaning}\\
    \midrule[0.5pt]
    $\mathcal{D}$ & Distribution of input data & $\theta$ & Learnable parameters \\
    $\mathcal{D}_t$ & Shifting distribution of input data at time $t$ & $\alpha$ & Learning rate \\
    $n$ & Number of unique categorical features & $g$ & Standard gradient without compression \\
    $d$ & Embedding dimension & $\widetilde{g}$ & Gradient in compressed DLRM \\
    $x_{cat}$ & Categorical feature  & $\mathcal{L}$ & Loss function \\
    $x_{num}$ & Numerical feature  & $\mathcal{M}$ & Memory usage (of the embedding table) \\
    $y$ & Ground truth label & $M$ & Memory budget \\
    $\hat{y}$ & Prediction & $CR$  & Compression ratio \\
    $E$ & Embedding table & $w$ & Number of buckets in HotSketch \\ 
    $E^{*}$ & Compressed embedding table & $c$ & Number of slots in each bucket in HotSketch \\ 
    $f$ & Neural network & $k$ & Number of hot features \\
    \bottomrule[0.5pt]
    \end{tabular}
    
\end{table}

\subsection{Embedding Compression}\label{sec:2.2}

Embedding compression is mainly conducted within a memory constraint.
Denoting $M$ as the memory budget of the embedding table, $\mathcal{M}$ as the memory function mapping an embedding table to corresponding memory usage, and $E^*$ as the compressed embedding table, the optimization of DLRM within a memory constraint is formulated as follows:

\begin{align}
\begin{split}
	\min\limits_{E^*,f} &\mathbb{E}_{(x_{cat}, x_{num} ,y)\sim \mathcal{D}}\mathcal{L}(y,f(E^*(x_{cat}),x_{num})), \\ 
	\mathrm{s.t.} &\mathcal{M}(E^*) \le M.
\end{split}
\end{align}

The memory function excludes the memory usage of neural networks since it is fixed and negligible compared to the memory usage of embedding tables.

We define compression ratio as the multiple of the original memory to the compressed memory, to reflect the degree of compression:
$CR = \frac{\mathcal{M}(E)}{\mathcal{M}(E^*)}$.
In practical applications, using a compression ratio of $10\times$ can reduce the cost of distributed deployment, $100\times$ to $1000\times$ can allow for single-device deployment, and an extreme compression ratio of $10000\times$ can enable DLRMs on edge devices.

For online training, the fixed data distribution $\mathcal{D}$ in the above definition can be modified to a variable distribution $\mathcal{D}_t$, which is continuously evolving over time $t$.

\presec
\section{\name Design}\label{sec:cafe}
\postsec

\subsection{\name Overview}

\bbb{Rationale:}
We design \name, an efficient embedding framework that is simultaneously compact, adaptive, and fast. 
The key idea of \name is to dynamically distinguish important features (called hot features) from unimportant ones (called non-hot features), and allocate more resources to hot features. 
Specifically, we define the importance score of a feature using the L2-norm of its gradient, which is proven to have good theoretical properties in Section~\ref{subsubsec:conv} and also in previous works~\cite{DBLP:conf/icml/Gopal16,DBLP:conf/icml/KatharopoulosF18,lai2023adaembed}.
We further experimentally demonstrate the effectiveness of gradient norms in Section~\ref{subsec:confsens}.
We observe that in most training data, the feature importance follows a highly skewed distribution, where most features have small importance scores and only a small fraction of hot features are very important. 
For example, Figure~\ref{fig:sys:zipf} illustrates that the feature importance distributions in Criteo dataset and CriteoTB dataset are highly consistent with Zipf distributions of parameters 1.05 and 1.1, respectively. 
Therefore, if we can allocate more memory to the embeddings of hot features and less memory to those of non-hot features, it is possible to significantly improve the model quality under the same memory usage of embedding tables.

\begin{figure}[htbp]
\centering
\subfigure[Criteo.]{
\scalebox{0.27}{
\includegraphics[width=\linewidth]{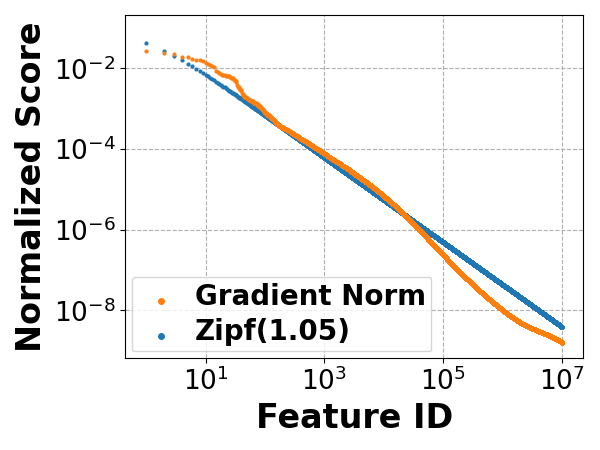}
}
\label{fig:sys:zipf:kaggle}
}
\subfigure[CriteoTB.]{
\scalebox{0.27}{
\includegraphics[width=\linewidth]{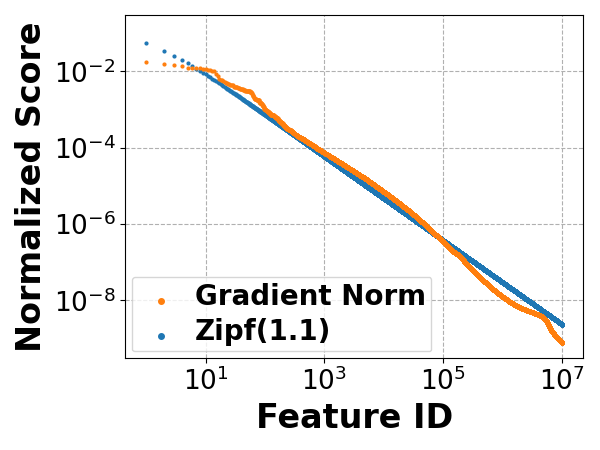}
}
\label{fig:sys:zipf:criteotb}
}
\caption{Comparing gradient norm and Zipf distributions.}
\label{fig:sys:zipf}
\end{figure}

As shown in Figure~\ref{fig:sys:overview}, in \name, we propose a novel sketch algorithm, called \algo, to capture feature importance and report top-$k$ hot features in real time (Section~\ref{sec:hotsketch}). 
In each training iteration, we first fetch data samples from the input training data, and query each feature from these samples in \algo.
For each feature, \algo reports its current importance score, and if its score exceeds a predefined threshold, we regard it as a hot feature. 
We then lookup the embeddings for hot and non-hot features respectively. 
In \name, for each hot feature, we allocate a unique embedding, and we store the pointer to this embedding in \algo. 
For the non-hot features, we use hash embedding tables where multiple features can share one embedding. 
We will discuss how to migrate embeddings between the tables of hot and non-hot features (Section~\ref{sec:migrate}). 
Guided by our design philosophy, we further propose a multi-level hash embedding framework to better embrace the skewed feature importance distribution (Section~\ref{sec:multilevel}).
Afterwards, we feed the embeddings into the downstream neural network for prediction and get the gradient norm for each feature. 
Finally, we update the importance of these features in \algo using their gradient norms.

\begin{figure}[htbp]
    \includegraphics[width=0.7\linewidth]{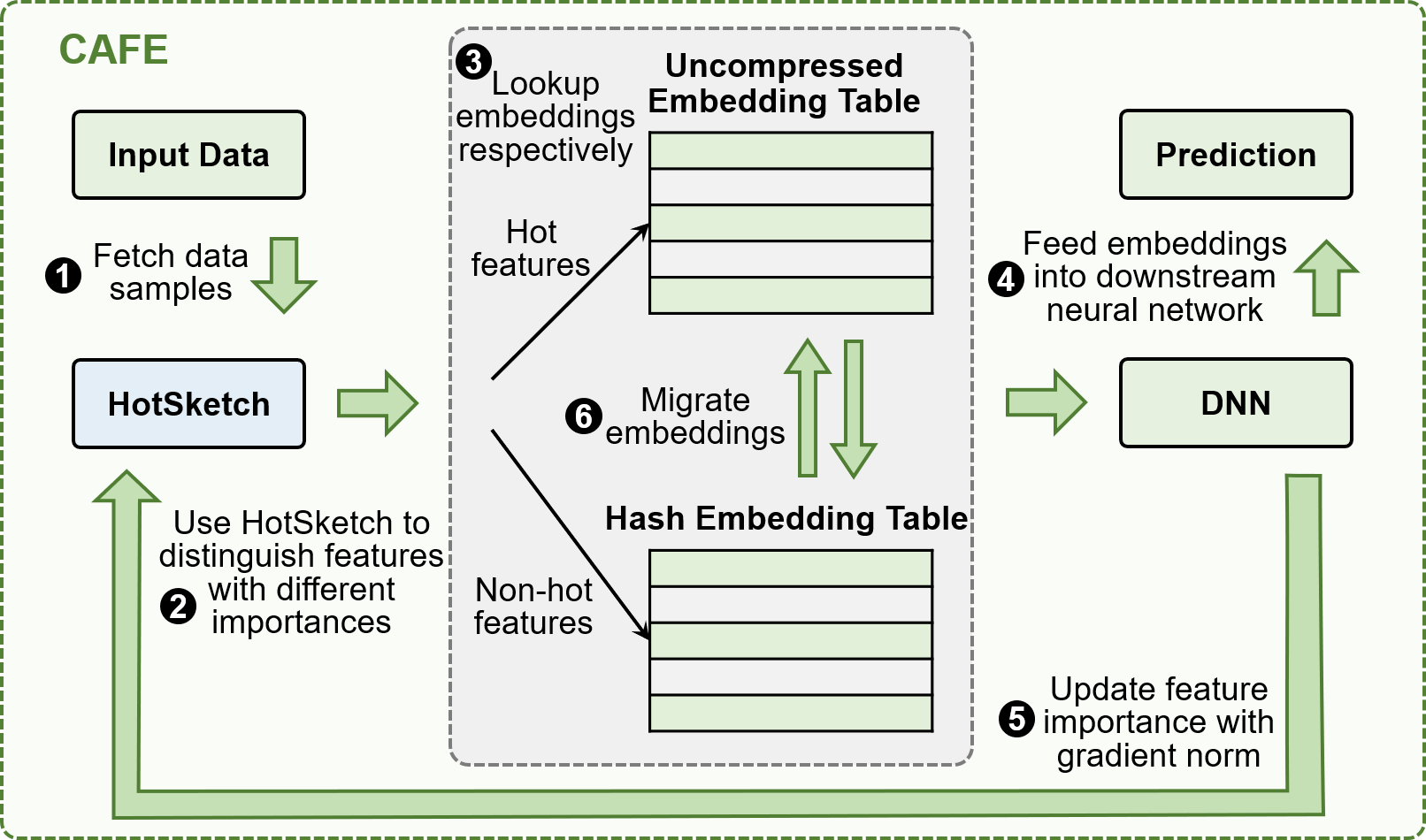}
    \caption{Overview of \name.}\label{fig:sys:overview}
\end{figure}

    \subsection{The \algo{} Algorithm}\label{sec:hotsketch}

\bbb{Rationale:}
We design HotSketch to capture hot features with high importance scores in a single pass, which is essentially a problem of finding top-$k$ frequent items (features) in streaming data. 
Currently, Space-Saving~\cite{spacesaving} is the most recognized algorithm for solving top-$k$ problem.
It maintains frequent items in sorted doubly linked list and uses a hash table to index this list. 
However, this hash table not only doubles the memory usage but also imposes time inefficiency due to numerous memory accesses caused by pointer operations.
Based on the idea of Space-Saving, we propose HotSketch, which removes the hash table while still maintaining the $O(1)$ time complexity. 
We theoretically prove that our HotSketch well inherits the theoretical results of Space-Saving (Section~\ref{subsubsec:theosketch}), and empirically validate the performance of HotSketch (Section~\ref{subsec:expalgo}).

\bbb{Data structure:}
As depicted in Figure~\ref{fig:algo}, HotSketch consists of an array of $w$ buckets $\mathcal{B}[1], \cdots, \mathcal{B}[w]$. 
We use a hash function $h(\cdot)$ to map each feature into one bucket. 
Each bucket contains $c$ slots. 
Each slot stores a feature ID and its importance score. 

\bbb{Insertion:}
For each incoming feature $f_i$ associated with an importance score $s_i$, we first calculate the hash function to locate a bucket $\mathcal{B}[h(f_i)]$, termed as the hashed bucket of $f_i$. 
Then, we check bucket $\mathcal{B}[h(f_i)]$ and encounter three possible scenarios:
(1) $f_i$ is recorded in $\mathcal{B}[h(f_i)]$.
We add $s_i$ to its importance score. 
(2) $f_i$ is not recorded in $\mathcal{B}[h(f_i)]$ and there exists an empty slot in $\mathcal{B}[h(f_i)]$. 
We insert $f_i$ into the empty slot by setting this slot to $(f_i, s_i)$. 
(3) $f_i$ is not recorded in $\mathcal{B}[h(f_i)]$ and $\mathcal{B}[h(f_i)]$ is full.
We locate the feature with the smallest score $(f_{min}, s_{min})$, replace $f_{min}$ with $f_i$, and add $s_i$ to $s_{min}$. 
In other words, we set the slot $(f_{min}, s_{min})$ to $(f_i, s_{min}+s_i)$. 
Figure~\ref{fig:algo} shows an example of insertion.

\bbb{Discussion:}
HotSketch has the following advantages: 
(1) HotSketch has fast insertion speed. It processes each feature in a one-pass manner and has an $O(1)$ time complexity. In addition, HotSketch avoids complicated pointer operations and has only one memory access. 
(2) HotSketch is memory-efficient. It does not store pointers, and there are no empty slots in HotSketch after a brief cold start.
(3) HotSketch is hardware-friendly and can be accelerated with multi-threading and SIMD, thereby achieving superior data parallelism.

\begin{figure}[htbp]
    \includegraphics[width=0.7\linewidth]{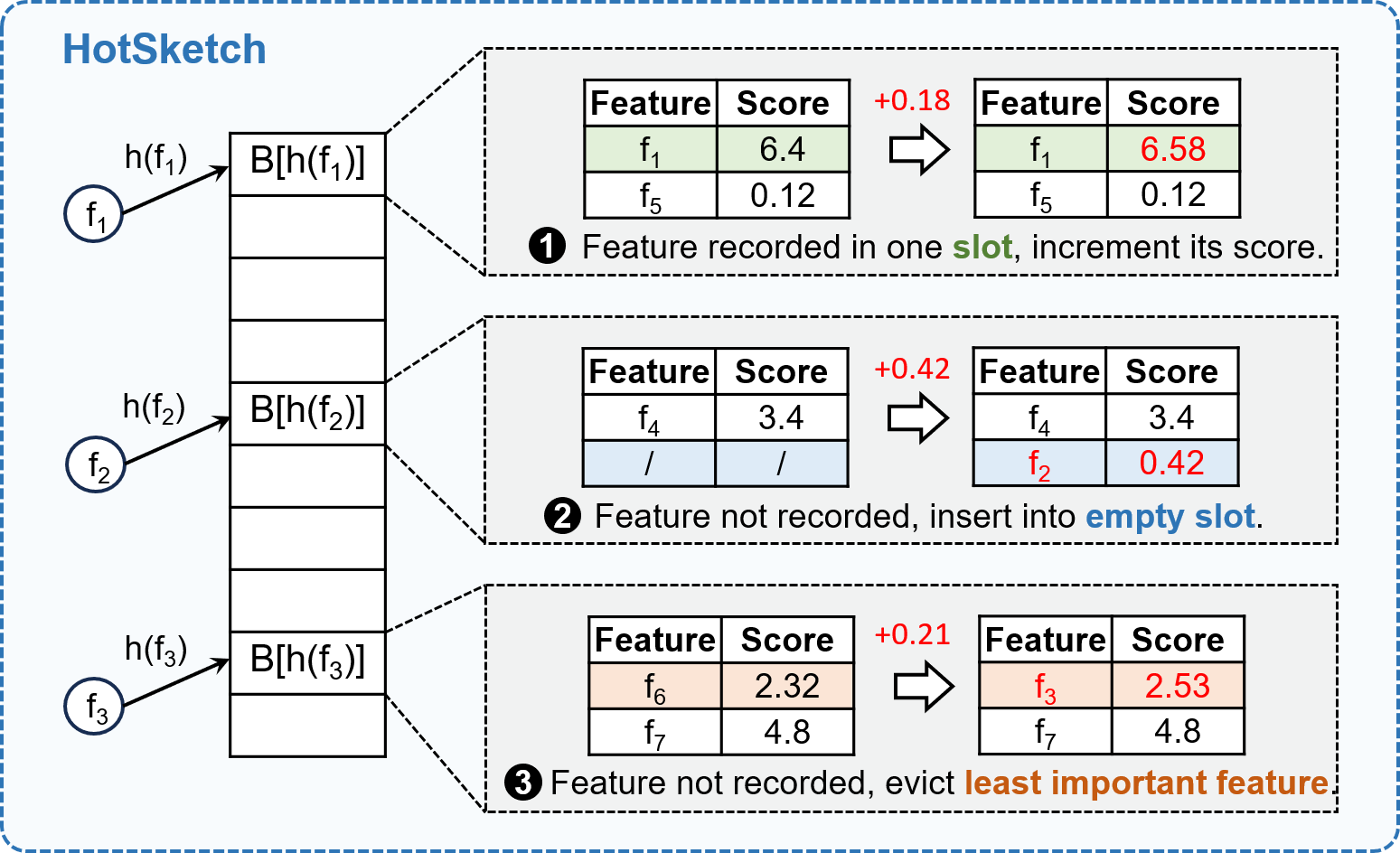}
    \caption{The \algo algorithm.}\label{fig:algo}
\end{figure}

    \subsection{Migration Strategy}\label{sec:migrate}

During the online training of DLRMs, the distribution of feature importance fluctuates with data distribution changes, meaning that the hot features are not constant throughout the training process.
Since \algo already records the feature importance during training, it can naturally support dynamic hot features by embeddings migration between uncompressed and hash embedding tables.

In \algo, we set a threshold to distinguish hot features, and the entry and exit of hot features occur throughout the training process.
Almost every feature that reaches \algo for the first time is considered a non-hot feature with a low importance score.
When a non-hot feature's importance score surpasses the threshold, it transitions into a hot feature, and its embedding migrates from the shared table to the uncompressed table as initialization, ensuring the feature's representation remains smooth throughout the training process.
Conversely, when a hot feature's importance score drops below the threshold (by eviction or decay), it becomes a non-hot feature, and its embedding is discarded from the uncompressed table.
Considering that the newly migrated non-hot feature is no longer important, its original exclusive embedding is simply ignored and the shared embedding is used instead.
The threshold is meticulously set, allowing \algo to always saturate with hot features and adapt to distribution changes.
If the importance scores alter rapidly, we decay the scores periodically.

During training, it's vital to maintain an appropriate migration frequency.
If the migration occurs too frequently, the learning process may not be smooth enough due to the replacement of embeddings, and the migration will generate substantial delay. 
Conversely, if the migration occurs too infrequently, \algo cannot capture changes in the distribution, leading to a decline in model quality.
By setting a suitable threshold in \algo, a moderate migration frequency can enable the model to adapt to changes in distribution without negatively impacting convergence and latency.

    \subsection{Multi-level Hash Embedding}\label{sec:multilevel}

In \algo, features are categorized into hot and non-hot features, with the latter outnumbering the former, considering typical compression ratios ranging from $10\times$ to $1000\times$.
A substantial number of non-hot features are treated identically in \algo, sharing a hash embedding table with the same rate of collisions.
Given that these features' importance scores also conform to a highly skewed Zipf distribution, it's logical to further segregate non-hot features based on their importance scores and assign different hash embedding tables to them.
Therefore, we integrate multi-level hash embedding, as shown in Figure~\ref{fig:multilevel}.

\begin{figure}[htbp]
    \includegraphics[width=0.7\linewidth]{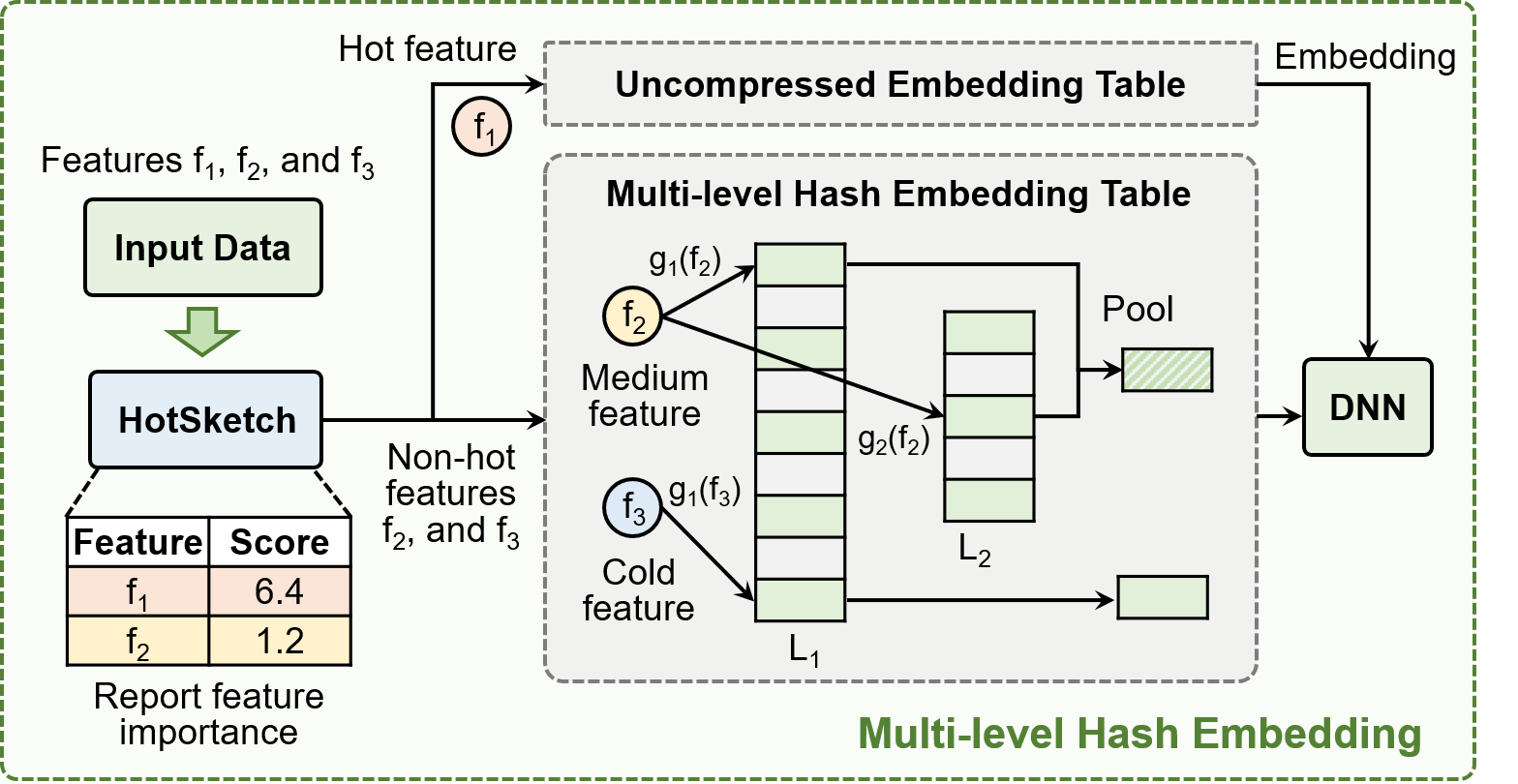}
    \caption{Overview of multi-level hash embedding.}\label{fig:multilevel}
\end{figure}

With multi-level hash embedding, we partition non-hot features into more refined categories of different importance levels, and assign to them different number of embeddings from multiple tables.
For simplicity, we focus on 2-level hash embedding, where non-hot features are divided into medium features and cold features.
We expand the functionality of \algo to identify medium features by estimating their importance scores.
Since medium features are more significant, they reference 2 embedding vectors from 2 distinct hash embedding tables, while cold features only look up a single embedding vector.
This design draws inspiration from prior hash-based methods~\cite{DBLP:conf/recsys/ZhangLXKTGMDPIU20} that also adopt multiple embedding vectors to enrich representations and boost model quality.

We illustrate the multi-level embedding process using an example in Figure~\ref{fig:multilevel}.
(1) Input features $f_1$, $f_2$, $f_3$ are fed into \algo.
Among them $f_1$ has a score larger than the hot threshold, $f_2$ has a score above the medium threshold, and $f_3$ has a score lower than the thresholds, so they are classified as hot, medium, and cold features respectively.
(2) Hot and cold features look up the embedding vectors as before.
(3) Medium feature $f_2$ looks up two embedding vectors from two hash embedding tables, and obtains the final embedding through a pooling process.
To ensure that the training process remains smooth, the hash function is combined with hash tables.
When a feature is migrated between middle and cold classes, it always retrieves the same embedding vector from the first embedding table.
For pooling operation, in practice, we find that simple summation of embeddings performs well, since a feature's embedding vectors are always updated in the same direction.

The design of the multi-level hash embedding is based on the observation that a unique embedding is a comprehensive representation with no information loss, whereas, for hash embeddings, the larger the number of embeddings involved, the fewer the collisions and the more information a feature can retain.
Through experiments detailed in Section~\ref{sec4.4}, we find that multi-level hash embedding performs better, with a reduction of $0.25\%$ in training loss and an increase of $0.08\%$ in testing AUC.

    \subsection{Theoretical Analysis}
\label{subsec:theo}

\subsubsection{\textbf{Accuracy of HotSketch}}\label{subsubsec:theosketch}
\
\newline
In this section, we theoretically analyze the performance of HotSketch in finding hot features. We derive the probability that a hot feature with a large importance score is recorded in HotSketch.

\begin{theorem}
Given a data stream with $n$ features, and suppose their importance score vector is $a=\{a_1, a_2,\cdots,a_n\}$, where $a_1\geqslant a_2\geqslant\cdots\geqslant a_n$.
    Suppose that our HotSketch has $w$ buckets, and each bucket contains $c$ cells.
    Without distribution assumption, for a hot feature with a total score larger than $\gamma\lVert a\rVert_1$, it can be held in HotSketch with probability at least:
    $
    \Pr>1-\frac{1-\gamma}{(c-1)\gamma w}.
    $
\end{theorem}

\begin{proof}
    The expected score sum of the other features $\hat{f}$ entering the same bucket is:
    $
    E\left[\hat{f}\right]=\frac{(1-\gamma)\lVert a\rVert}{w}.
    $
    
    By following the properties of SpaceSaving algorithm, if the score $\hat{f}$ of the other features entering the bucket is no more than $(c-1)\gamma\lVert a\rVert_1$, then the feature must be held in the bucket.
    Using Markov inequality, we have
    $
    \Pr\left(\hat{f}>(c-1)\gamma\lVert a\rVert_1\right)\leqslant\frac{1-\gamma}{(c-1)\gamma w},
    $
    which means that
    $
    \Pr>1-\frac{1-\gamma}{(c-1)\gamma w}.
    $
\end{proof}

\begin{lemma}
    Given a data stream with score vector $a=\{a_1, a_2,\cdots,a_n\}$, where $a_1\geqslant a_2\geqslant\cdots\geqslant a_n$.
    Suppose that $a$ follows a Zipfian distribution with parameter $z$, meaning that $a_i=\frac{a_1}{i^z}$.
    Suppose our HotSketch has $w$ buckets, and each bucket contains $c$ cells.
    Suppose we would like to check whether the $k'$ hottest features can be hashed into the buckets.
    Then the mathematical expectation of the score sum of the non-hot features entering each bucket is:
    $
    E[\hat{f}]\leqslant\frac{\lVert a\rVert_1\cdot k'^{1-z}}{w}
    $
    with probability at least $3^{-\frac{k'}{w}}$ for $z>1$ and $n\to+\infty$.
\end{lemma}

\begin{proof}
    The probability that the $k'$ hottest features are not hashed into this bucket is:
    $
    \left(1-\frac{1}{w}\right)^{k'}
    =
    \left(\left(1-\frac{1}{w}\right)^w\right)^{\frac{k'}{w}}
    >
    3^{-\frac{k'}{w}}
    $.
    
    When $w\geqslant 6$, $\left(1-\frac{1}{w}\right)^w$ increases monotonically with $w$.
    The expected score sum of the non-hot features entering this bucket is:
    \begin{align*}
    E[\hat{f}]
    = &
    \frac{\sum_{i=k'+1}^{n} a_i}{w}
    =
    \frac{\sum_{i=k'+1}^{n} \frac{a_1}{i^z}}{w}
    =
    \frac{\lVert a\rVert_1}{w} \cdot \left(\sum_{i=k'+1}^{n} i^{-z}\right) \cdot \frac{1}{\sum_{i=1}^{n} i^{-z}}
    \\
    \leqslant&
    \frac{\lVert a\rVert_1}{w} \cdot \left(\int_{k'}^{+\infty} x^{-z}dx\right) \cdot \left(\int_{1}^{+\infty} x^{-z}dx\right)^{-1}
    \\
    \leqslant&
    \frac{\lVert a\rVert_1}{w} \cdot \frac{k'^{1-z}}{z-1}\cdot (z-1)
    =
    \frac{\lVert a\rVert_1 k'^{1-z}}{w}
    \end{align*}
    for $z>1$ and $n\to+\infty$.
\end{proof}

\begin{theorem}
    Given a data stream with score vector $a=\{a_1, a_2,\cdots,a_n\}$, where $a_1\geqslant a_2\geqslant\cdots\geqslant a_n$.
    Suppose that $a$ follows a Zipfian distribution with parameter $z$.
    Suppose that our HotSketch has $w$ buckets, and each bucket contains $c$ cells.
    Let $k'=\eta w$.
    Then for a hot feature with a score larger than $\gamma\lVert a\rVert_1$, it can be held in the sketch with probability at least:
    $
    \Pr>\mathop{sup}\limits_{\eta>0}\left(3^{-\eta}\cdot\left(1-\frac{\eta}{(c-1)\gamma (\eta w)^z}\right)\right)
    $
    for $z>1$ and $n\to+\infty$.
    \label{theo:sketch:p}
\end{theorem}

\begin{proof}
    The condition $\mathcal{C}$ that none of the $k'$ hottest features collide with this item holds with probability at least $3^{-\frac{k'}{w}}$.

    By following the properties of SpaceSaving algorithm, if the scores $\hat{f}$ of the other features entering the bucket is no more than $(c-1)\gamma\lVert a\rVert_1$, then the feature must be held in the bucket.
    
    Using Markov inequality, we have
    $$
    \Pr{\left(\hat{f}> (c-1)\gamma\lVert a\rVert_1~|~\mathcal{C}\right)}\leqslant\frac{\frac{\lVert a\rVert_1\cdot k'^{1-z}}{w}}{(c-1)\gamma\lVert a\rVert_1}
    =
    \frac{k'^{1-z}}{(c-1)\gamma w}.
    $$
    Then we have
    \begin{align*}
    \Pr{\left(\hat{f}> (c-1)\gamma\lVert a\rVert_1\right)}
    \leqslant &
    \Pr{\left(\hat{f}> (c-1)\gamma\lVert a\rVert_1,\mathcal{C}\right)}+\Pr{(\neg\mathcal{C})}
    \\
    \leqslant &
    3^{-\frac{k'}{w}}\cdot\left(\frac{k'^{1-z}}{(c-1)\gamma w}-1\right)+1.
    \end{align*}
    Let $k'=\eta w$, we have
    $$
    \Pr{\left(\hat{f}> (c-1)\gamma\lVert a\rVert_1\right)}
    \leqslant   
    3^{-\eta}\cdot\left(\frac{1}{\eta^{z-1}(c-1)\gamma w^z}-1\right)+1.
    $$
    And we have the probability that this feature must be held greater than
    $$
    \Pr>\mathop{sup}\limits_{\eta>0}\left(3^{-\eta}\cdot\left(1-\frac{\eta}{(c-1)\gamma (\eta w)^z}\right)\right).
    $$
    
\end{proof}

\begin{corollary}
    The larger the parameter $c$, $w$, $z$, and $\gamma$, the larger the probability that the feature with score larger than $\gamma\lVert a\rVert_1$ be held in the sketch.
    The larger $c$ and $w$ means the larger memory used by sketch, the larger $z$ means the more skew the data stream is, and the larger $\gamma$ means the hotter the feature is.
\end{corollary}

\begin{proof}
    The following formula  monotonically decreases with parameter $c$, $w$, $z$, and $\gamma$:
    $
    \frac{\eta}{(c-1)\gamma (\eta w)^z}.
    $
\end{proof}

\begin{corollary}
    To let the feature with score larger than $\gamma\lVert a\rVert_1$ be held with maximum probability in a fixed memory budget, the more skew the data stream is, the less cells per bucket should be used.
    Specifically, we recommend to use $c=1+\frac{1}{z-1}$.\label{corollary:optimal}
\end{corollary}

\begin{proof}
    With a fixed memory budget $M=cw$, to minimize $\frac{\eta}{(c-1)\gamma (\eta w)^z}$, we should maximize
    $
    (c-1)w^z=\left(\frac{M}{w}-1\right)w^z.
    $
    
    As it has a derivative function 
    $$
    \left[\left(\frac{M}{w}-1\right)w^z\right]'=((z-1)M-zw)w^{z-2},
    $$
    the optimal $w$ should be $\frac{z-1}{z}M$, and the optimal $c^*$ should be 
    $$
    c^*=\frac{z}{z-1}=1+\frac{1}{z-1}.
    $$
\end{proof}

\bbb{Discussion:}
From Corollary~\ref{corollary:optimal}, we can see that under fixed memory usage ($M=cw$), the optimal $c$ is affected by data distribution. 
Under non-skewed data distribution (small $z$), we should use larger $c$ and smaller $w$ to better approximate the results of basic Space-Saving.  
Under highly skewed data distribution (large $z$), we should use smaller $c$ and larger $w$ to lower the impact of hash collisions between hot features.  
This might be because under highly skewed data, using small $c$ can already guarantee us to find hot features with high probability. 
In this scenario, the performance of HotSketch is mainly affected by hash collisions between hot features. 
We surprisingly find that this corollary is consistent with our experimental results in Figure~\ref{fig:exp:sketch:recall}.

\begin{figure}[htbp]
    \includegraphics[width=0.5\linewidth]{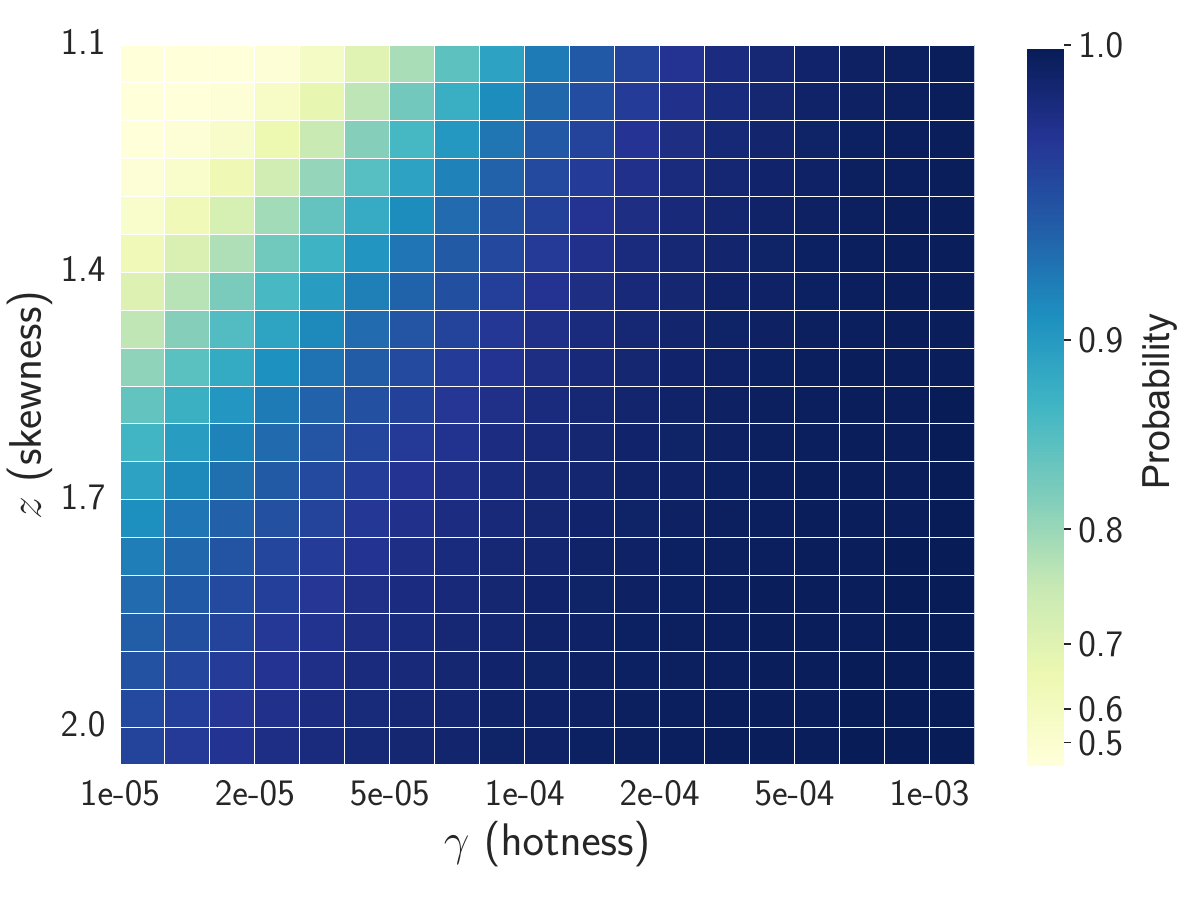}
    \caption{Numerical analysis for the probability of HotSketch identifying hot features, where the x-axis represents the hotness of the feature and the y-axis represents the skewness of the feature hotness distribution (Theorem~\ref{theo:sketch:p}).}\label{fig:heatmap}
\end{figure}

\bbb{\textit{Experimental analysis (Figure~\ref{fig:heatmap}):}}
Although we cannot directly obtain the analytical solution of $\Pr$ from Theorem~\ref{theo:sketch:p},  we can give the numerical solution of $\Pr$ under different
$\gamma$ and $z$ by numerical simulation. 
In our simulation, we set $w=10000$ and $c=4$.
We can see that larger $z$ goes with higher $\Pr$, showing that HotSketch is more suitable for capturing top-$k$ features on skewed data distribution.  
In addition, larger $\gamma$ also goes with higher $\Pr$, showing that hotter features have larger probability of being captured by HotSketch. 
The results are consistent with our design goal.

\subsubsection{\textbf{Convergence Analysis against Deviation}}\label{subsubsec:conv}
\
\newline
As mentioned in Section~\ref{sec:intro:limitation}, in hash-based methods, there will be deviations that can hinder the convergence of embeddings.
\name aims to minimize the deviation of embedding gradients, which indeed reflects the deviation of embedding parameters. In this section, we analyze how this deviation affects the convergence of SGD algorithm. We study the following (non-convex) empirical risk minimization problem:

$$\min\limits_{\theta\in\mathbb{R}^D} f(\theta)=\frac1N \sum\limits_i^N f_i(\theta), \theta_{t+1}=\theta_t-\alpha\widetilde{g}_{i_t}$$

where $\alpha$ is learning rate, $g_{i_t}=\nabla f_i(\theta_{i_t})$ is the standard gradient without compression, $\widetilde{g}_{i_t}$ is the practical gradient with compression. We make the assumptions below following~\cite{DBLP:conf/icml/Allen-Zhu17,DBLP:conf/icml/FuHHJSZ020}.

\bbb{Assumptions.} For $\forall i\in\{1,2,...,N\}, \theta, \theta'\in\mathbb{R}^D$, we make the following assumptions:

(1. $L$-Lipschitz) $\|\nabla f_i(\theta)-\nabla f_i(\theta')\| < L\|\theta-\theta'\|$ ;

(2. Bounded moment) $\mathbb{E}[\|\nabla f_i(\theta)\|] < \sigma_0$,  $\mathbb{E}[\|\nabla f(\theta)\|] < \sigma_0$;

(3. Bounded variance) $\mathbb{E}[\|\nabla f_i(\theta)-\nabla f(\theta)\|] < \sigma$ ;

(4. Existence of global minimum)  $\exists f^* s.t. f(\theta)\ge f^*$ .

\begin{theorem}
    Suppose we run SGD optimization with \name on DLRMs satisfying the assumptions above, with $\epsilon_t = \|\widetilde{g}_{i_t}-g_{i_t}\|$ as the deviation of embedding gradients. Assume the learning rate $\alpha$ satisfies $\alpha < \frac1{L}$. After $T$ steps, for $\overline{\theta}_T$ which is randomly selected from $\{\theta_0, \theta_1, ..., \theta_{T-1}\}$, we have:
    
    $$\mathbb{E}[\|\nabla f(\overline{\theta}_T)\|^2]\le \frac{f(\theta_0)-f^*}{T\alpha(1-\alpha L)} + \frac{\alpha(2L\sigma^2 + \sigma_0^2)}{2(1-\alpha L)} + \frac{(1+\alpha^2 L)\sum_{t=0}^{T-1}\mathbb{E}[\epsilon^2_t]}{2T\alpha(1-\alpha L)}$$
\end{theorem}

The proof is in the supplementary file on our GitHub page~\cite{cafecode}. 

As $T$ increases, with a proper learning rate $\alpha = O(\frac1{\sqrt{T}})$, the first two terms at the right hand side of above inequality tend to 0, and the convergence of SGD is mainly influenced by the deviation $\epsilon_t$. In the scenario of compression, there is no bound for this deviation; yet the design of \name is proposed to minimize the deviation.

Since we assign those importance features with exclusive embedding parameters, their parameters have little deviation; for features sharing embeddings with each other, the deviation is introduced by the hash collisions. 
Generally, we cannot directly obtain the deviation of gradients, but according to the $L$-Lipschitz assumption, the deviation of gradients is bounded by the deviation of weights.
As non-hot features share embeddings with each other, the deviation of weights comes from other features' gradients, which may disturb the learning direction.
Based on this observation, \name choose to use gradient norm as the importance of features.
For less important features, their gradient norms are relatively small, which limits the deviation of weights to some extent.

\section{Implementation}

We implement \name as a plug-in embedding layer module based on PyTorch. It can directly replace the original Embedding module in any PyTorch-based recommendation models with minor changes. Usage examples can be found on our GitHub page~\cite{cafecode}. We consider extending \name to other frameworks (TensorFlow, Hetu~\cite{DBLP:journals/chinaf/MiaoNZZC23}, etc.) in the future.

\noindent\bbb{\name Backend:} 
We implement the \algo algorithm in C++ to reduce the overall latency, and implement the rest of \name using PyTorch operators.
For \algo, we set the number of bucket in \algo to the pre-determined number of hot features, with 4 slots per bucket.
We use one sketch structure for all feature fields instead of one sketch per field, because the distribution of hot features across fields is unclear, which is better handled directly with importance scores.

\noindent\bbb{Fault Tolerance:} 
We register all \algo's states as buffers in \name's PyTorch module, so that the states can be saved and loaded alongside model parameters. 
This simple design requires no additional modifications and enables DLRM with \name to use checkpoints for training and inference. 
When training resumes with checkpoints, parameters and states are reloaded simultaneously.

\noindent\bbb{Memory Management:}
We place the whole \algo structure on CPU, since it is not compute-intensive. 
Built upon PyTorch operators, \name's embedding module can run on any accelerators (including CPU, GPU) where PyTorch is supported.

\presec
\section{Experimental Results}\label{sec:exp}
\postsec

In this section, we conduct experiments on four widely used recommendation datasets and compare \name with existing embedding memory compression methods. 
We experimentally show that \name satisfies all three requirements. We also design experiments to reflect the effectiveness of \algo.

    \subsection{Experimental Settings}

\subsubsection{Models and Datasets.}
We conduct the experiments on three representative recommendation models DLRM~\cite{DBLP:journals/corr/abs-1906-00091} \footnote{In this section, we use the term "DLRM" to refer to this specific model, rather than the abbreviation of general "Deep Learning Recommendation Model".}, WDL~\cite{DBLP:conf/recsys/Cheng0HSCAACCIA16}, and DCN~\cite{DBLP:conf/kdd/WangFFW17}. These models are popular in both academia and industry. 
All models follow the architecture discussed in Section~\ref{sec:2.1}, with slight differences in the neural network part. In DLRM, a cross layer performs dot operations between embeddings, producing cross terms for subsequent fully-connected (FC) layers; in WDL, embeddings are fed into a wide network (1 FC layer) and a deep network (several FC layers), and finally the results are summed together for predictions; in DCN, cross layers multiply the embeddings with their projected vectors, producing element-level cross terms for subsequent FC layers.
Since our method is essentially an embedding layer plugin, the conclusions can be generalized to other recommendation models with little effort.
We set the configurations of the models as in the original paper.

\begin{table}[ht]
    \small
    \centering
    \caption{Overview of the datasets.}

    \begin{tabular}{cccccc}
    \toprule
    \textbf{Dataset}&\textbf{\#Samples}&\textbf{\#Features}&\textbf{\#Fields}&\textbf{Dim}&\textbf{\#Param}\\
    \midrule
    Avazu & 40,428,967 & 9,449,445 & 22 & 16 & 150M \\
    Criteo & 45,840,617 & 33,762,577 & 26 & 16 & 540M \\
    KDD12 & 149,639,105 & 54,689,798 & 11 & 64 & 3.5B \\
    CriteoTB & 4,373,472,329 & 204,184,588 & 26 & 128 & 26B 
    \\
    \bottomrule
    \end{tabular}

    \label{tab:datasets}
\end{table}

We train on three large-scale datasets Avazu~\cite{avazu}, Criteo~\cite{criteokaggle}, KDD12~\cite{kdd12}, and an extremely large-scale dataset CriteoTB~\cite{criteotb}. Criteo Kaggle Display Advertising Challenge Dataset (Criteo)~\cite{criteokaggle} and Criteo Terabytes Click Logs (CriteoTB)~\cite{criteotb} contain 7 and 24 days of ads click-through rate (CTR) prediction data respectively, which are adopted in MLPerf~\cite{mlperf}. 
Each data sample has 13 numerical fields and 26 categorical fields. 
For CriteoTB, we set the field's maximum cardinality to $4e7$, the same as in the MLPerf configuration.
Avazu Click-Through Rate Prediction Dataset (Avazu)~\cite{avazu} and KDD Cup 2012, Track 2 (KDD12)~\cite{kdd12} are another two widely-used CTR datasets. They have no numerical field. Avazu contains 10 days of CTR data with 22 categorical fields. KDD12 has no temporal information, and has 11 categorical fields.
For each dataset, we use the appropriate embedding dimension based on the benchmarks~\cite{DBLP:journals/corr/abs-1906-00091,DBLP:conf/cikm/ZhuLYZH21} or our experiments on the uncompressed models.
The statistics of the datasets are listed in Table~\ref{tab:datasets}. Since the numerical field is not our focus, we omit it from the table.
In Section~\ref{sec:expr:manual}, we construct a new dataset with a more significant shift in data distribution to further validate \name's ability to adapt to changes in data distribution.

\subsubsection{Baselines.}
We compare \name with Hash Embedding~\cite{DBLP:conf/icml/WeinbergerDLSA09}, Q-R Trick~\cite{DBLP:conf/kdd/ShiMNY20}, and AdaEmbed~\cite{lai2023adaembed}. 
Hash embedding is a simple baseline using only one hash function, providing a lower bound for all compression methods. 
Q-R Trick is an improved hash-based method, using multiple hash functions and complementary embedding tables to reduce the overall collisions. 
AdaEmbed is an adaptive method, recording all features' importance scores and dynamically allocates embedding vectors only for important features.
We also compare with uncompressed embedding tables.
In Section~\ref{sec:expr:end2end:mde}, we compare \name with a column compression method MDE~\cite{DBLP:conf/isit/GinartNMYZ21}.
If not specified, the hyperparameters of the baselines are the same as in the original paper or code.

\subsubsection{Hardware Environment.}
We conduct all experiments on NVIDIA RTX TITAN 24 GB GPU cards.
Since we focus on embedding compression with large compression ratios, we do not incur distributed training or inference.

\subsubsection{Metrics.}
We employ training loss and testing AUC (area under the ROC curve) to measure model quality.
Specifically, we use the data samples except the last day as the training set, and the data samples of the last day as the testing set.
We use the testing AUC on the last day as the metric for offline training, and the average loss during training as the metric for online training.
We train one epoch on the training set in chronological order, which is common in industry.
Since KDD12 has no temporal information, we randomly shuffle the data and select 90\% for training and the rest for testing.
For memory usage, besides embedding tables, we also consider the memory of additional structures to achieve a fair judgment on memory efficiency.
We use latency and throughput to measure the speed of each method.

\subsection{End-to-end Comparison}
\label{subsec:end-to-end}

In this section, we compare \name with baseline methods in an end-to-end manner.
For large-scale datasets, we train with compression ratios ranging from $2\times$ to $10000\times$, while for the CriteoTB dataset, we train with compression ratios ranging from $10\times$ to $10000\times$, ensuring the model fits in the memory.

\subsubsection{Metrics v.s. Compression Ratios.}
We conduct the main experiments on DLRM.
The testing AUC and the training loss of Criteo and CriteoTB under different compression ratios are plotted in Figure~\ref{fig:exp:end2end:mem}, representing the performance of offline and online training respectively.
For KDD12, we only plot the testing AUC in Figure~\ref{fig:exp:end2end:datasets:kdd12_auc} since it does not contain temporal information for online training. 
For Avazu, given the significant changes in distributions between days as shown in Figure~\ref{fig:intro:kl}, we focus on the online training performance and plot the training loss in Figure~\ref{fig:exp:end2end:datasets:avazu_loss}.
Only \name and Hash can compress the embedding tables to extreme $10000\times$ compression ratio, while Q-R Trick can only compress to around $500\times$ due to its complementary index design, and AdaEmbed can only compress to $5\times$ in Avazu and Criteo with dimension 16, $20\times$ in KDD12 with dimension 64, and $50\times$ in CriteoTB with dimension 128.
Compared to Hash and Q-R Trick, \name is always closer to ideal result that uses uncompressed embedding tables, showing excellent memory efficiency.
When varying the compression ratio, on Criteo dataset \name improves the testing AUC by $1.79\%$ and $0.55\%$ compared to Hash and Q-R Trick respectively on average; on CriteoTB dataset the improvement is $1.304\%$ and $0.427\%$; on KDD12 dataset the improvement is $1.86\%$ and $3.80\%$.
\name also reduces the training loss by $2.31\%$, $0.72\%$ on Criteo dataset, $1.35\%$, $0.59\%$ on CriteoTB dataset, and $3.34\%$, $0.76\%$ on Avazu dataset compared to Hash and Q-R Trick, exhibiting better performance for both offline and online training.
The training loss of Hash fluctuates with the increase of CR on KDD12, which may be due to the instability of the Hash method and a certain degree of randomness in its embedding sharing.
The improvement of \name over Hash is greater with larger compression ratio.
Compared to Hash, at $10000\times$ compression ratio, \name improves 
$3.92\%$, $3.68\%$, and $5.16\%$ testing AUC on Criteo, CriteoTB, KDD12; \name reduces $4.61\%$, $3.24\%$, and $11.21\%$ training loss on Criteo, CriteoTB, Avazu.
Compared to AdaEmbed, \name reaches nearly the same testing AUC and training loss on Criteo dataset, achieves an increase of $0.04\%$ testing AUC and a decrease of $0.12\%$ training loss on CriteoTB dataset, achieves an increase of $0.82\%$ testing AUC on KDD12 dataset and a decrease of $0.83\%$ training loss on Avazu dataset.
AdaEmbed can distinguish hot features with no errors, but it uses much memory for storing importance information of all features, with less memory for embedding vectors compared to \name, leading to comparable results at small compression ratios.

\begin{figure}[htbp]
\centering
\subfigure[AUC v.s. CR on Criteo.]{
\scalebox{0.23}{
\includegraphics[width=\linewidth]{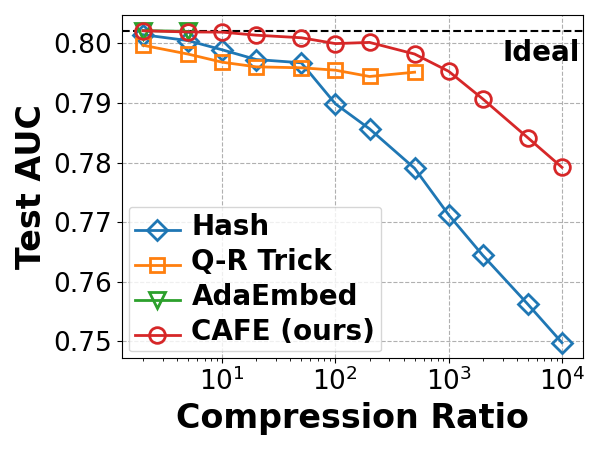}
}
\label{fig:exp:end2end:mem:kaggle_auc}
}
\subfigure[AUC v.s. CR on CriteoTB.]{
\scalebox{0.23}{
\includegraphics[width=\linewidth]{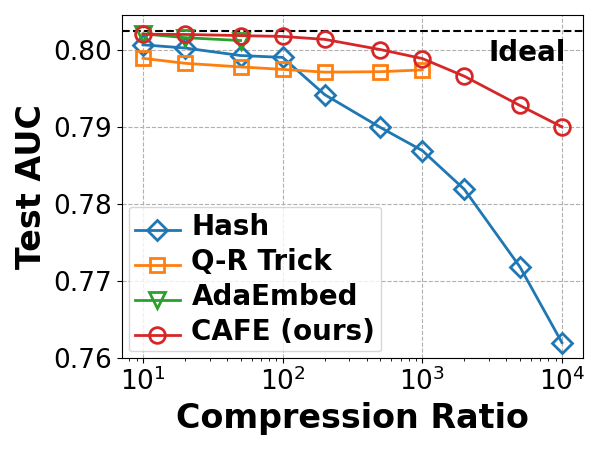}
}
\label{fig:exp:end2end:mem:criteotb_auc}
}
\subfigure[Loss v.s. CR on Criteo.]{
\scalebox{0.23}{
\includegraphics[width=\linewidth]{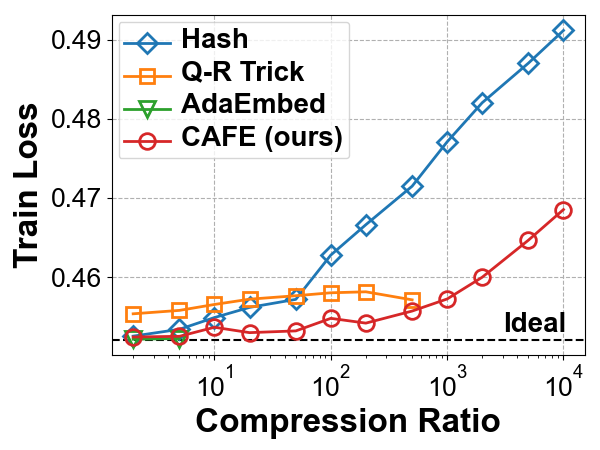}
}
\label{fig:exp:end2end:mem:kaggle_loss}
}
\subfigure[Loss v.s. CR on CriteoTB.]{
\scalebox{0.23}{
\includegraphics[width=\linewidth]{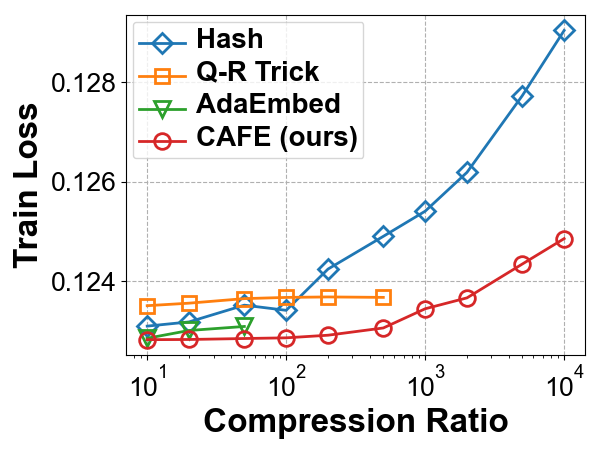}
}
\label{fig:exp:end2end:mem:criteotb_loss}
}
\caption{Metrics v.s. compression ratios.}
\label{fig:exp:end2end:mem}
\end{figure}

\begin{figure}[htbp]
\centering
\subfigure[\scriptsize{AUC v.s. iter (Criteo $100\times$).}]{
\scalebox{0.23}{
\includegraphics[width=\linewidth]{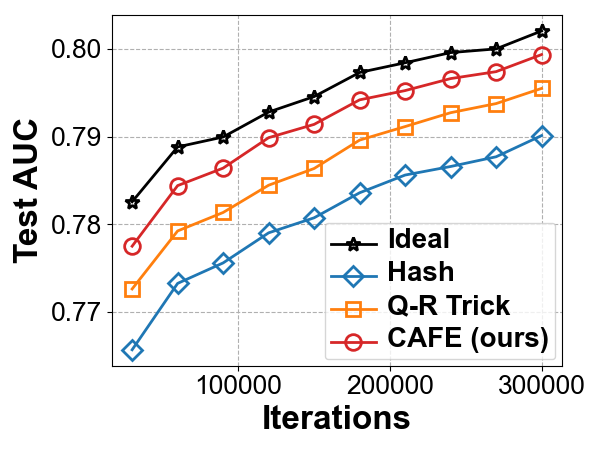}
}
\label{fig:exp:end2end:iter:kaggle_auc}
}
\subfigure[\scriptsize{AUC v.s. iter (CriteoTB $100\times$).}]{
\scalebox{0.23}{
\includegraphics[width=\linewidth]{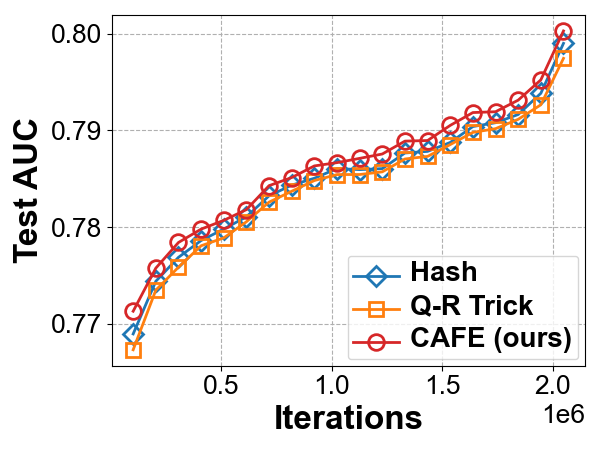}
}
\label{fig:exp:end2end:iter:criteotb_auc}
}
\subfigure[\scriptsize{AUC v.s. iter (Criteo $5\times$).}]{
\scalebox{0.23}{
\includegraphics[width=\linewidth]{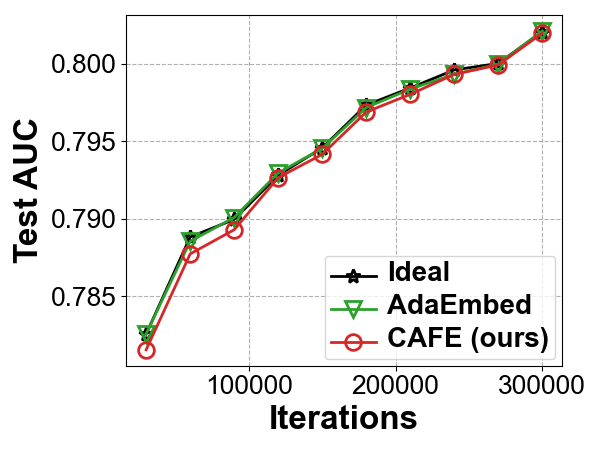}
}
\label{fig:exp:end2end:iter:kaggle_ada_auc}
}
\subfigure[\scriptsize{AUC v.s. iter (CriteoTB $50\times$).}]{
\scalebox{0.23}{
\includegraphics[width=\linewidth]{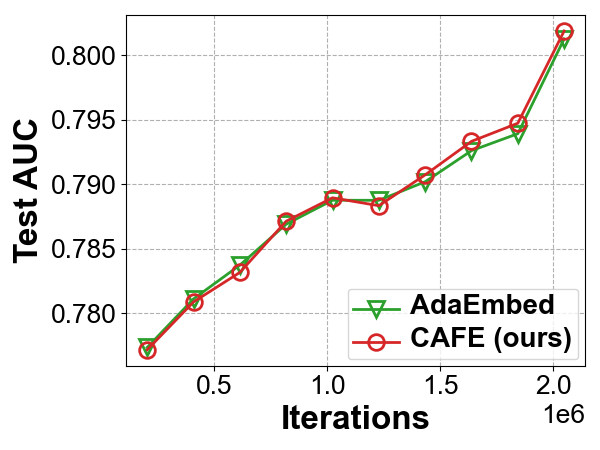}
}
\label{fig:exp:end2end:iter:criteotb_ada_auc}
}

\subfigure[Loss v.s. iterations on Criteo ($100\times$).]{
\scalebox{0.47}{
\includegraphics[width=\linewidth]{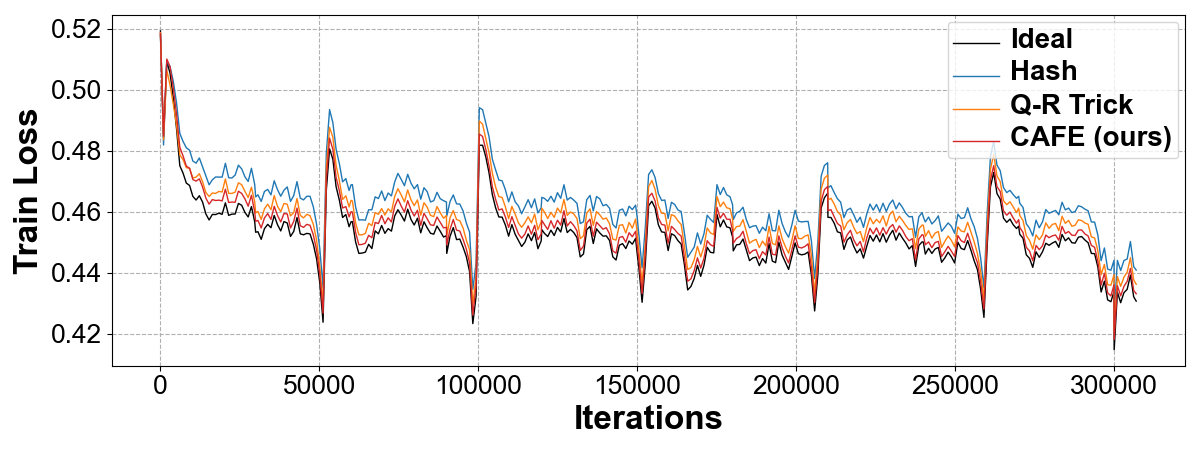}
}
\label{fig:exp:end2end:iter:kaggle_loss}
}
\subfigure[Loss v.s. iterations on CriteoTB ($100\times$).]{
\scalebox{0.47}{
\includegraphics[width=\linewidth]{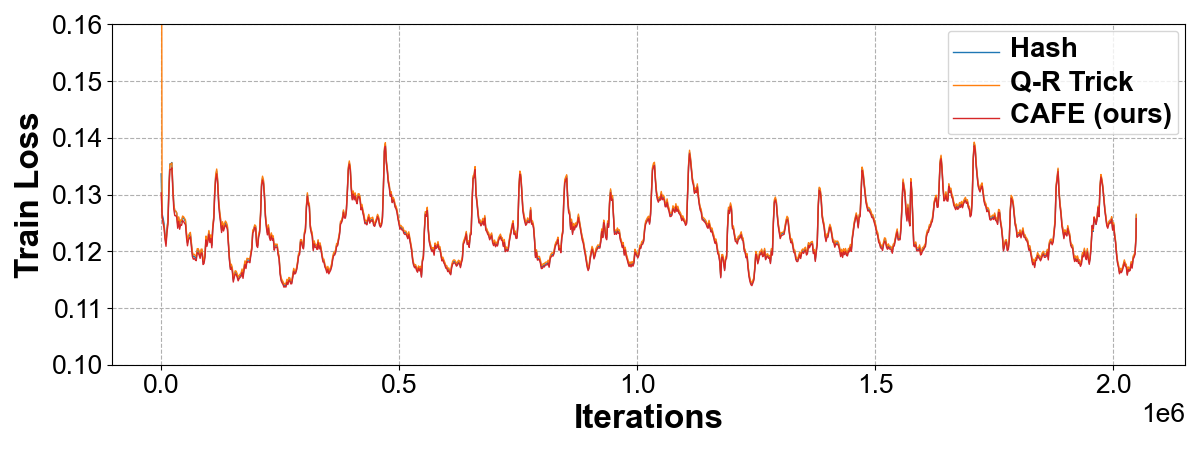}
}
\label{fig:exp:end2end:iter:criteotb_loss}
}

\subfigure[Loss v.s. iterations on Criteo ($5\times$).]{
\scalebox{0.47}{
\includegraphics[width=\linewidth]{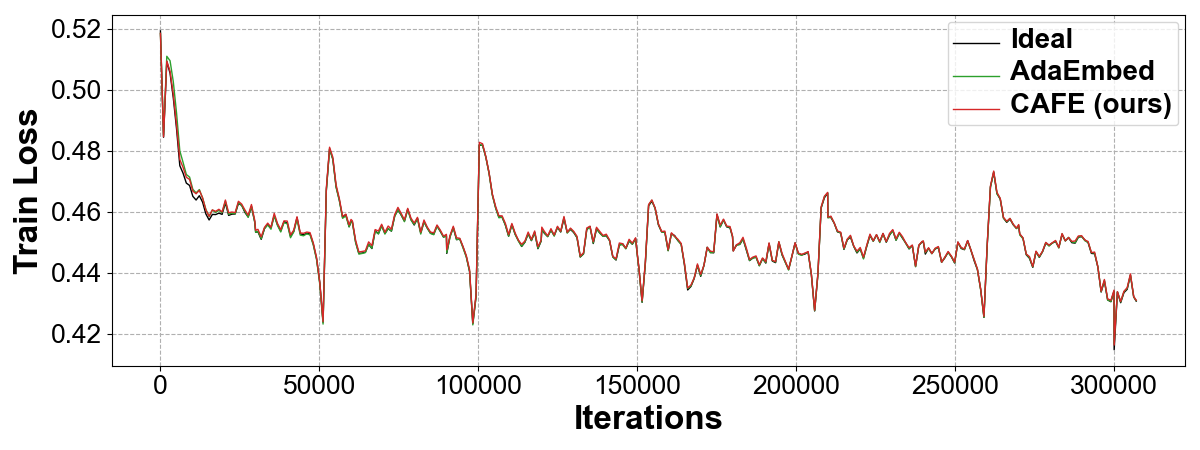}
}
\label{fig:exp:end2end:iter:kaggle_ada_loss}
}
\subfigure[Loss v.s. iterations on CriteoTB ($50\times$).]{
\scalebox{0.47}{
\includegraphics[width=\linewidth]{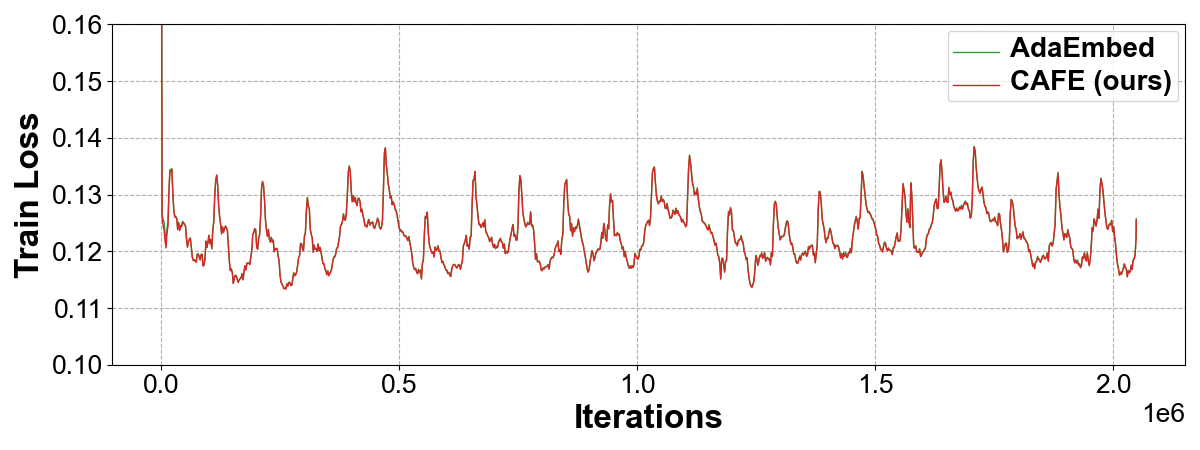}
}
\label{fig:exp:end2end:iter:criteotb_ada_loss}
}
\caption{Metrics v.s. iterations.}
\label{fig:exp:end2end:iter}
\end{figure}

\begin{figure}[htbp]
\centering
\subfigure[AUC v.s. CR on KDD12.]{
\scalebox{0.23}{
\includegraphics[width=\linewidth]{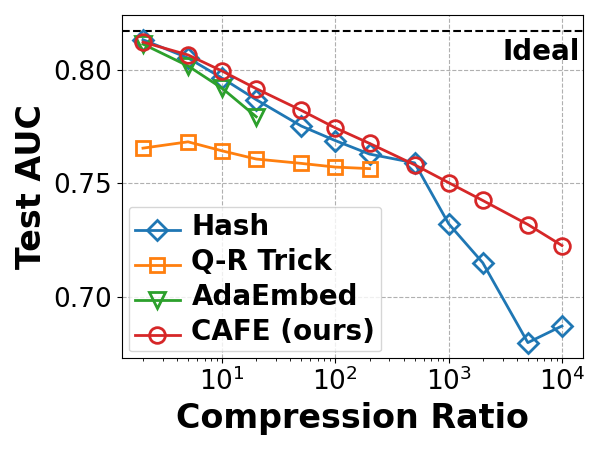}
}
\label{fig:exp:end2end:datasets:kdd12_auc}
}
\subfigure[Loss v.s. CR on Avazu.]{
\scalebox{0.23}{
\includegraphics[width=\linewidth]{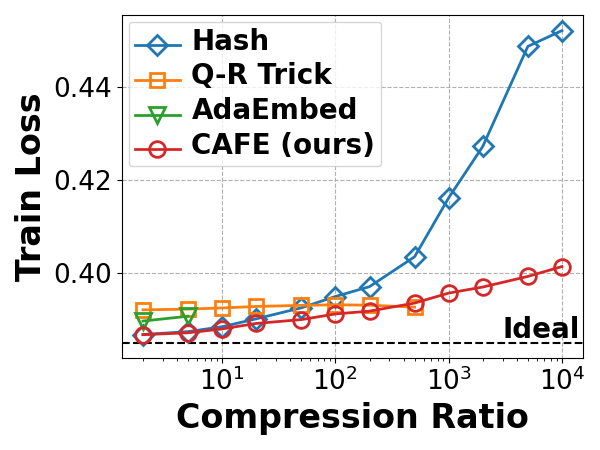}
}
\label{fig:exp:end2end:datasets:avazu_loss}
}
\subfigure[Loss v.s. iterations on Avazu ($5\times$).]{
\scalebox{0.46}{
\includegraphics[width=\linewidth]{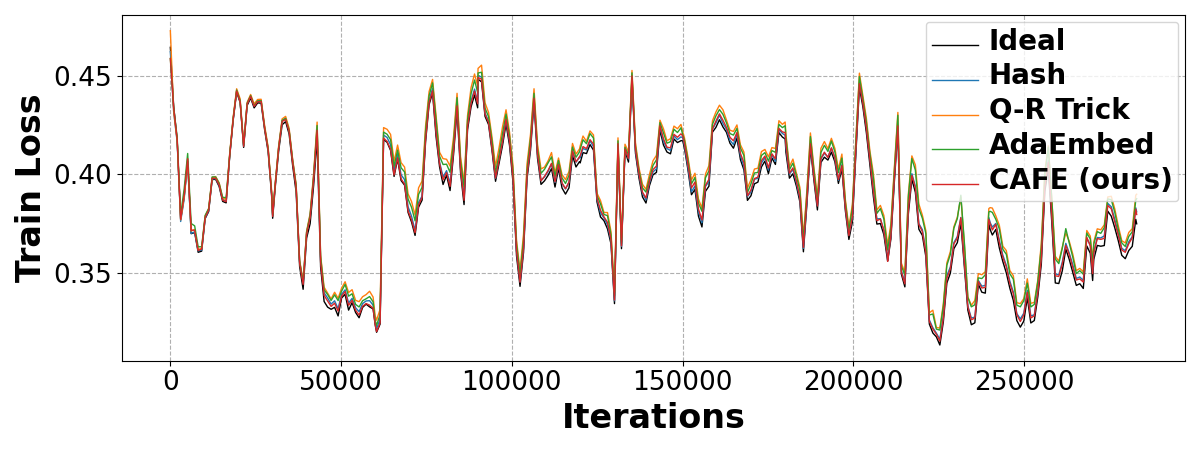}
}
\label{fig:exp:end2end:datasets:avazu_iter}
}
\caption{Performance on KDD12 and Avazu.}
\label{fig:exp:end2end:datasets}
\end{figure}

\begin{figure}[htbp]
\centering
\subfigure[WDL, AUC v.s. CR.]{
\scalebox{0.23}{
\includegraphics[width=\linewidth]{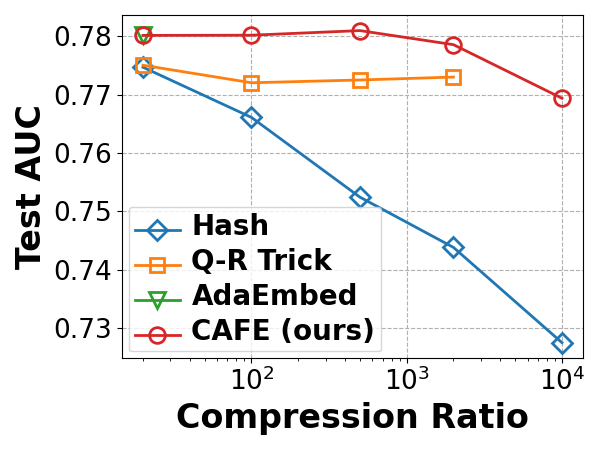}
}
\label{fig:exp:end2end:wdl:auc}
}
\subfigure[WDL, loss v.s. CR.]{
\scalebox{0.23}{
\includegraphics[width=\linewidth]{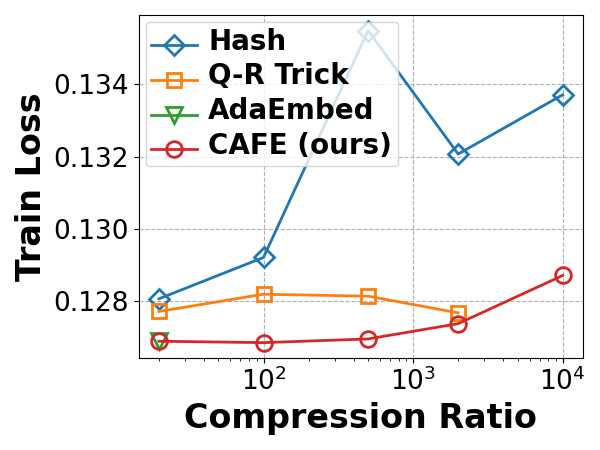}
}
\label{fig:exp:end2end:wdl:loss}
}
\subfigure[DCN, AUC v.s. CR.]{
\scalebox{0.23}{
\includegraphics[width=\linewidth]{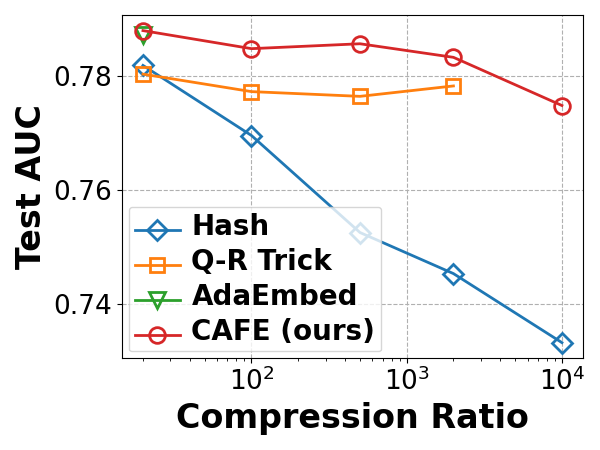}
}
\label{fig:exp:end2end:dcn:auc}
}
\subfigure[DCN, loss v.s. CR.]{
\scalebox{0.23}{
\includegraphics[width=\linewidth]{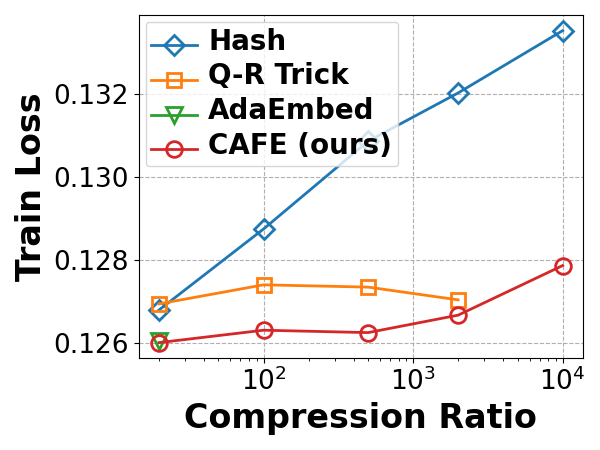}
}
\label{fig:exp:end2end:dcn:loss}
}
\caption{WDL and DCN performance on CriteoTB.}
\label{fig:exp:end2end:wdldcn}
\end{figure}

\subsubsection{Metrics v.s. Iterations.}
We check the convergence process of different methods.
Figure~\ref{fig:exp:end2end:iter} shows the metrics on Criteo and CriteoTB throughout iterations during training. 
Figure~\ref{fig:exp:end2end:datasets:avazu_iter} shows the training loss on Avazu throughout iterations.
We do not plot uncompressed embeddings trained on CriteoTB because the model cannot be held in our limited memory space.
In Figure~\ref{fig:exp:end2end:iter:kaggle_auc}-\ref{fig:exp:end2end:iter:criteotb_ada_auc}, the testing AUC curves tend to increase because the model continues to learn during training and the data distribution gradually approaches the distribution of the last day testing data.
\name has consistently better AUC during training compared to Hash and Q-R Trick.
However, \name does not show better performance at the beginning of training compared to AdaEmbed, mainly because \name has a cold-start process to populate \algo, where all features are initially non-hot features.
As training progresses, \name gradually achieves an AUC comparable to or better than AdaEmbed.
In Figure~\ref{fig:exp:end2end:iter:kaggle_loss}-\ref{fig:exp:end2end:iter:criteotb_ada_loss}, and \ref{fig:exp:end2end:datasets:avazu_iter}, the training loss curves fluctuate due to changes in data distributions.
\name always has a closer training loss to ideal result than Hash and Q-R Trick on Criteo and Avazu datasets, showing better online training ability.
The training curves of \name and AdaEmbed roughly coincide, since they are both designed for online training.
The CriteoTB dataset is large enough to adequately train various methods, resulting in the loss curves of different methods being indistinguishable.

\subsubsection{Experiments on WDL and DCN}

We use another two models, WDL~\cite{DBLP:conf/recsys/Cheng0HSCAACCIA16} and DCN~\cite{DBLP:conf/kdd/WangFFW17}, to experiment on the extremely large-scale dataset CriteoTB. The results are shown in Figure~\ref{fig:exp:end2end:wdldcn}. Similar to DLRM, \name consistently outperforms Hash and Q-R Trick at different compression ratios in both testing AUC and training loss. 
AdaEmbed is the most advanced compression method for small compression ratios, and \name achieves comparable performance to AdaEmbed. The training loss of Hash is not stable in WDL, possibly due to the instability of the Hash method itself and a certain degree of randomness in its embedding sharing.

\subsubsection{Comparison with Column Compression}\label{sec:expr:end2end:mde}

We also compare \name with MDE~\cite{DBLP:conf/isit/GinartNMYZ21}, a method that compresses columns of embedding tables instead of rows as in \name and other baselines. It introduces frequency information to allocate different embedding dimensions for different features, and then uses a trainable matrix to project the raw embeddings to the same final dimension. Since MDE does not compress the rows, and each feature needs to be assigned at least one dimension, the overall compression ratio is limited by the original embedding dimension. We plot the results in Figure~\ref{fig:exp:end2end:mde}. We also include a simple row compression baseline Hash for comparison. MDE’s performance is comparable to Hash on Criteo, but it drops dramatically on CriteoTB. 
To reduce the number of projection matrices, MDE simply uses the feature cardinality of the field to derive the frequency instead of using the actual frequency, which does not effectively utilize important features. It also significantly reduces the rank of the embedding matrix at large compression ratios, causing the embedding to lose semantic information. According to the experimental results, \name always outperforms MDE. 

\begin{figure}[htbp]
\centering
\subfigure[AUC v.s. CR on Criteo.]{
\scalebox{0.23}{
\includegraphics[width=\linewidth]{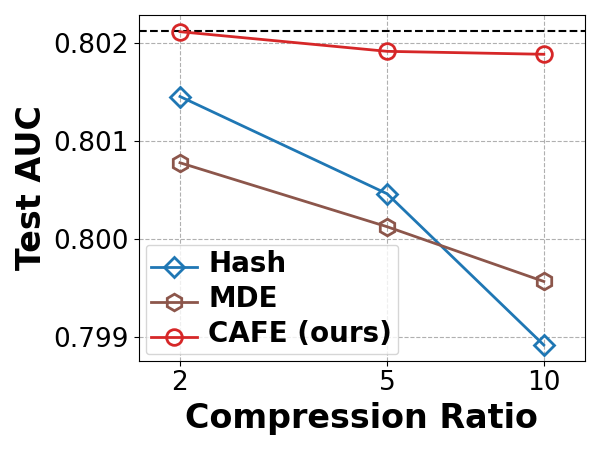}
}
\label{fig:exp:end2end:mde:kaggle_auc}
}
\subfigure[AUC v.s. CR on CriteoTB.]{
\scalebox{0.23}{
\includegraphics[width=\linewidth]{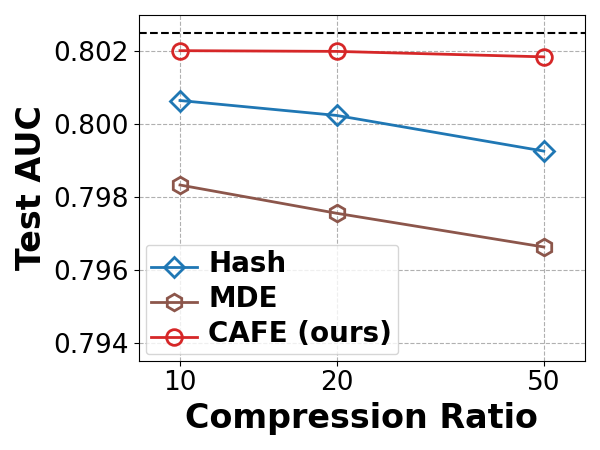}
}
\label{fig:exp:end2end:mde:criteotb_auc}
}
\subfigure[Loss v.s. CR on Criteo.]{
\scalebox{0.23}{
\includegraphics[width=\linewidth]{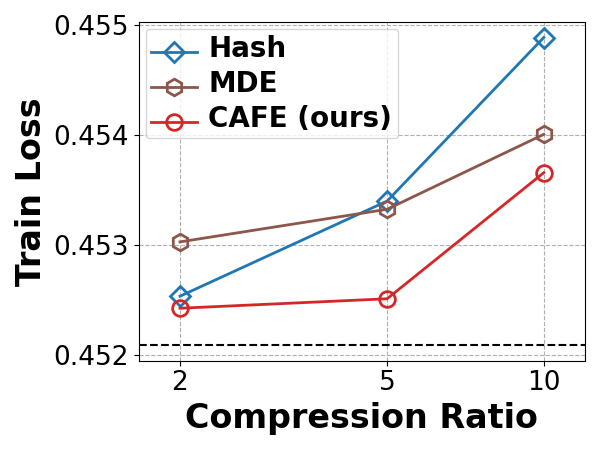}
}
\label{fig:exp:end2end:mde:kaggle_loss}
}
\subfigure[Loss v.s. CR on CriteoTB.]{
\scalebox{0.23}{
\includegraphics[width=\linewidth]{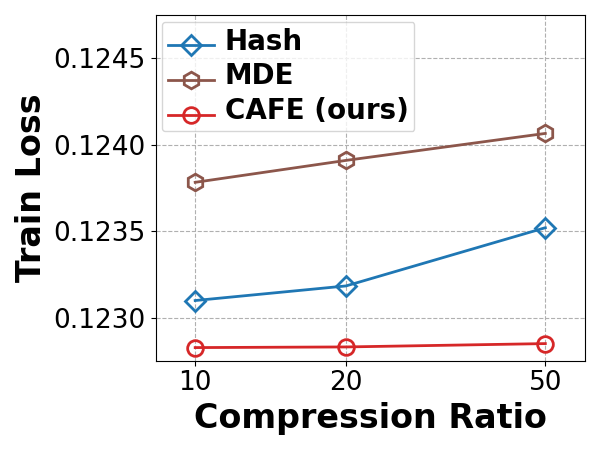}
}
\label{fig:exp:end2end:mde:criteotb_loss}
}
\caption{Comparison with MDE.}
\label{fig:exp:end2end:mde}
\end{figure}


\begin{figure}[htbp]
\centering
\subfigure[Latency.]{
\scalebox{0.47}{
\includegraphics[width=\linewidth]{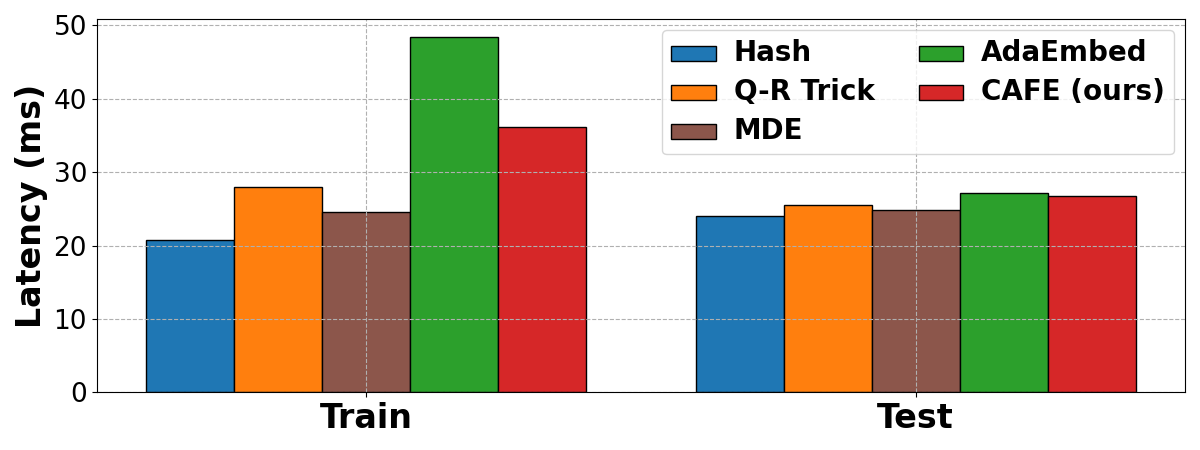}
}
}
\subfigure[Throughput.]{
\scalebox{0.47}{
\includegraphics[width=\linewidth]{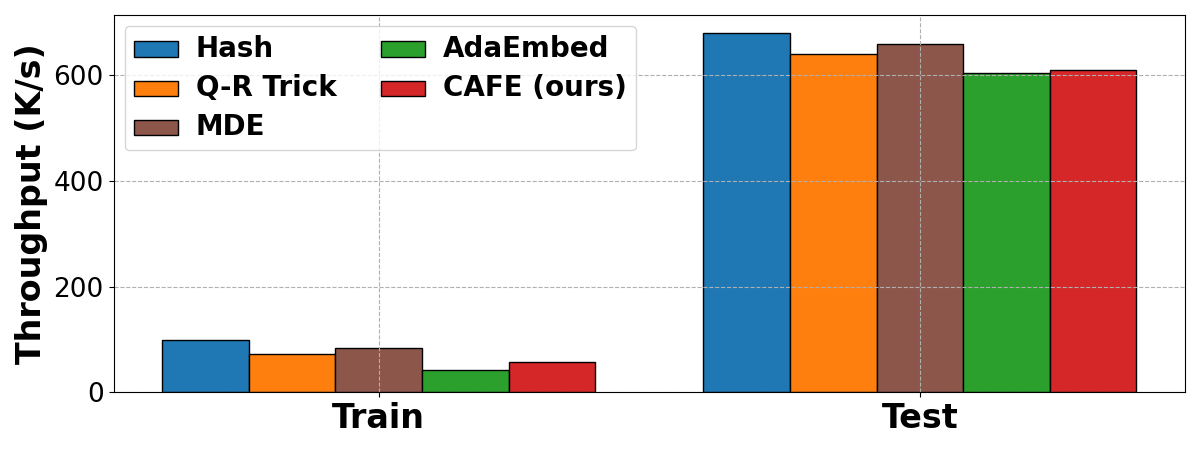}
}
}
\caption{Latency and throughput on CriteoTB ($10\times$).}
\label{fig:exp:end2end:throughput}
\end{figure}

\subsubsection{Latency and Throughput.}
We test the latency and throughput of each method in Figure~\ref{fig:exp:end2end:throughput}.
The experiments are conducted on CriteoTB dataset with a compression ratio of $10\times$.
We use 2048 batch size for training and 16384 batch size for inference, which is common in real applications.
As the data loading time and the DNN computing time is the same across different methods, the difference lies in the operations of the embedding layer.
Hash requires only an additional modulo operation compared to uncompressed embedding operations, and is therefore the fastest method in both training and inference.
Q-R Trick is also fast, because it only additionally introduces hash processes and the aggregation of embedding vectors.
Although MDE introduces matrix multiplication, it requires fewer memory accesses to obtain the embedding parameters, resulting in low latency and high throughput.
AdaEmbed and \name incur reallocation or migration of embeddings, which are inevitable for dynamic adjustments, leading to higher latency and lower throughput.
AdaEmbed regularly samples thousands of data to determine whether to reallocate, which introduces a large time overhead.
In contrast, \name determines the migration in \algo with negligible time overhead.
Compared to AdaEmbed, \name has $33.97\%$ lower training latency and $1.20\%$ lower inference latency.
Through the further experimental results in Section~\ref{subsec:expalgo}, we can see that \algo's $O(1)$ operation time only accounts for a small fraction of the overall time.

\subsubsection{Comparison with Offline Separation.}
We also compare \name with an offline feature separation version on Criteo dataset.
The offline separation version collects all data and makes statistics, separates hot and non-hot features according to frequency instead of gradient norms, and assigns embedding tables respectively.
It uses the same embedding memory as in \name for hot and non-hot features.
As shown in Figure~\ref{fig:exp:end2end:offline:auc}, two versions achieve nearly the same testing AUC under several compression ratios.
Compared to \name, the offline version has no errors in distinguishing hot features, but it can only use frequency, resulting in comparable performance.
Figure~\ref{fig:exp:end2end:offline:auc_iter} and Figure~\ref{fig:exp:end2end:offline:loss_iter} show the testing AUC and the training loss throughout the training process at $1000\times$ compression ratio.
At the beginning of training, the offline version has better testing AUC and training loss, because \name has a cold-start process to fill in the slots.
When the training process becomes stable, the two training loss curves almost completely coincide.
The offline version, however, cannot be used in practical applications.
First, it cannot adapt to online training, where the frequency information is unknown without recording.
Second, in offline training, memory storage and additional data traversal process are required for statistics, causing much overhead.
In contrast, \name naturally supports online and offline training without storing all importance information, so it can be directly applied in the industry.

\begin{figure}[htbp]
\centering
\subfigure[AUC v.s. CR.]{
\scalebox{0.23}{
\includegraphics[width=\linewidth]{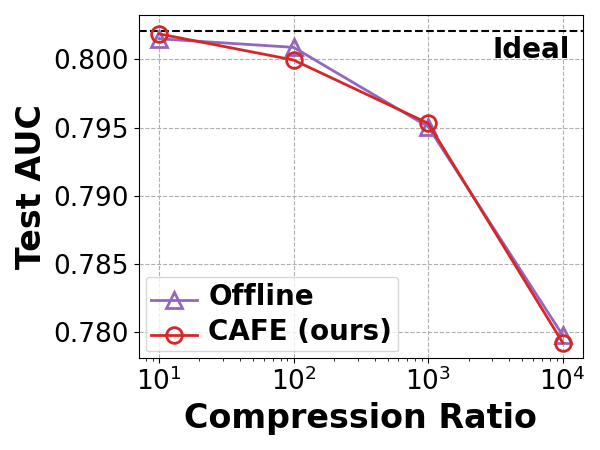}
}
\label{fig:exp:end2end:offline:auc}
}
\subfigure[AUC v.s. iter ($1000\times$).]{
\scalebox{0.23}{
\includegraphics[width=\linewidth]{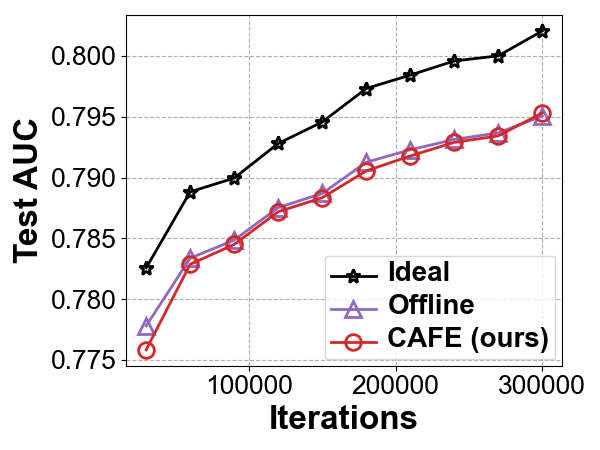}
}
\label{fig:exp:end2end:offline:auc_iter}
}
\subfigure[Loss v.s. iterations ($1000\times$).]{
\scalebox{0.46}{
\includegraphics[width=\linewidth]{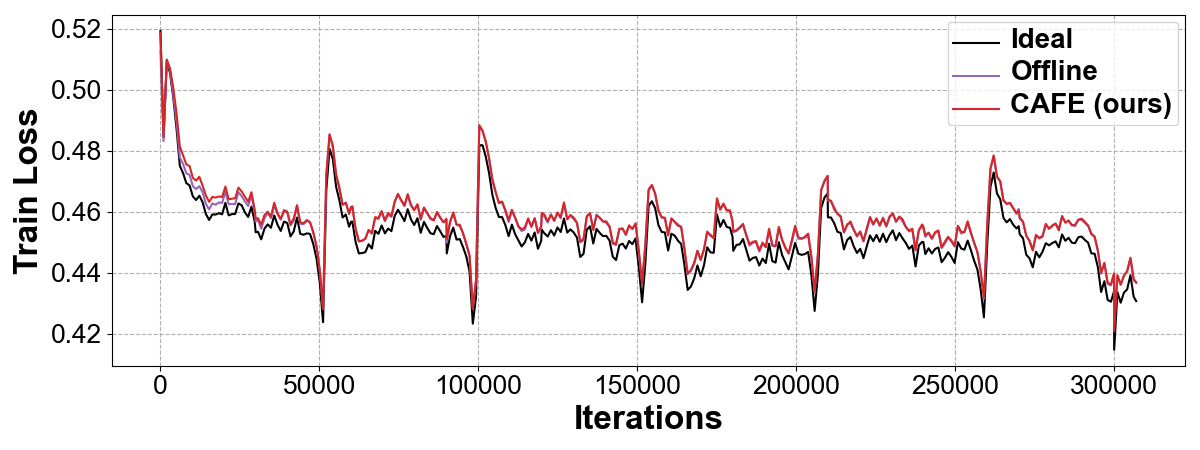}
}
\label{fig:exp:end2end:offline:loss_iter}
}
\caption{\name v.s. offline feature separation on Criteo dataset.}
\label{fig:exp:end2end:offline}
\end{figure}

\begin{figure}[htbp]
\centering
\subfigure[Memory for hot features.]{
\scalebox{0.23}{
\includegraphics[width=\linewidth]{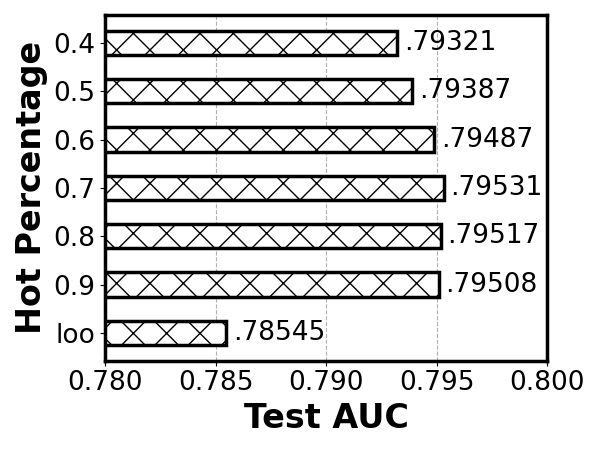}
}
\label{fig:exp:conf:toppercent}
}
\subfigure[Threshold of hot features.]{
\scalebox{0.23}{
\includegraphics[width=\linewidth]{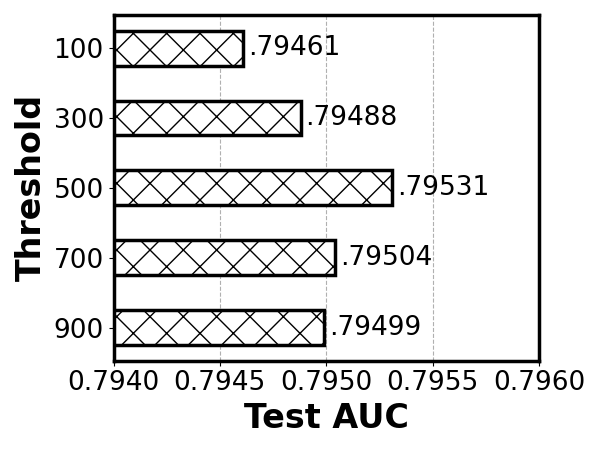}
}
\label{fig:exp:conf:threshold}
}
\subfigure[Decay of scores.]{
\scalebox{0.23}{
\includegraphics[width=\linewidth]{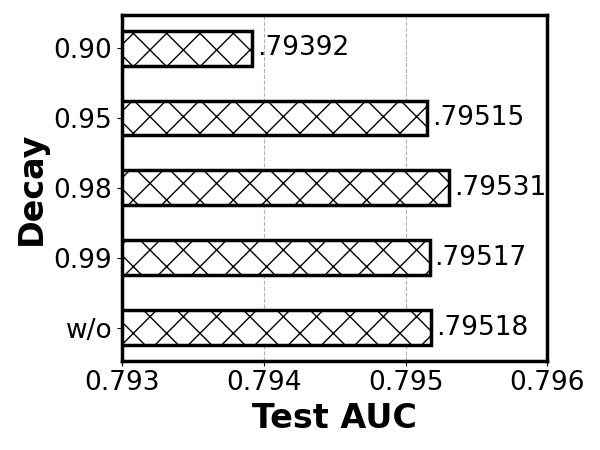}
}
\label{fig:exp:conf:decay}
}
\subfigure[Design details.]{
\scalebox{0.23}{
\includegraphics[width=\linewidth]{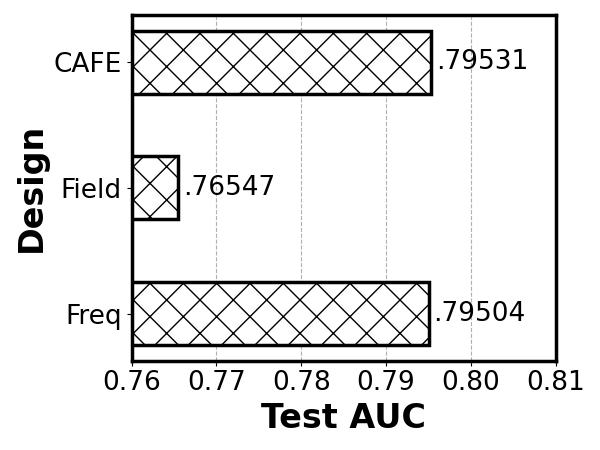}
}
\label{fig:exp:conf:ablation}
}
\caption{Experiments of configuration sensitivity on Criteo dataset ($1000\times$).}
\label{fig:exp:conf}
\end{figure}

\subsection{Configuration Sensitivity}\label{subsec:confsens}
In this section, we study the impact of configurations in \name.
We test different configurations on the Criteo dataset with a fixed compression ratio of $1000\times$, as shown in Figure~\ref{fig:exp:conf}.

\bbb{Memory for hot features.}
Given a limited memory constraint, we need to split the memory into three parts: sketch structure, hot features, and non-hot features.
We define the term "hot percentage" as the percentage of memory occupied by sketch structure and hot features, while the rest is used for non-hot features.
Since \algo stores 4 times the slots of the number of hot features, with each slot 3 attributes, the ratio of memory usage between \algo and $d$ dimension exclusive embeddings is $12:d$.
In the Criteo dataset, the dimension is set to 16, so \algo occupies $3/7$ of memory in hot percentage.
Figure~\ref{fig:exp:conf:toppercent} shows the testing AUC under different hot percentages, where “loo” means “leave-one-out”, leaving only one embedding for non-hot features.
A small hot percentage has low memory overhead of \algo, and allocate more memory for non-hot features, while a large hot percentage allocate more memory for hot features.
As hot percentage gradually increases from 0.4 to 1, the testing AUC first rises then drops.
When the hot percentage is small, enlarging hot percentage enables more hot features, contributing to model quality; when the hot percentage reaches 0.7, \name reaches the best testing AUC; when the hot percentage exceeds 0.7, \algo brings much overhead, and collisions of non-hot features increase dramatically, making the testing AUC drop.
At the extreme case "leave-one-out", all the non-hot features share only one embedding, leading to very bad model performance.
In practice, we find that setting hot percentage to around 0.7 is good enough for nearly all compression ratios.

\bbb{Threshold of hot features.}
Hot features are distinguished in \algo if their importance scores exceed the threshold.
We test different thresholds, and the experimental results are shown in Figure~\ref{fig:exp:conf:threshold}.
The testing AUC is bad when the threshold is set too high or too low.
If the threshold is set too high, the memory space allocated for hot features cannot be saturated, resulting in waste of memory and more non-hot features sharing hash embeddings.
If the threshold is set too low, the entry and exit of features will be too frequent, leading to unstable training process.
When threshold is set to 500, \name reaches the best model AUC.

\bbb{Decay of scores.}
The decay coefficient in \algo determines the exit of features.
All the importance scores in \algo, after a certain number of iterations, will be multiplied the decay coefficient to adapt to temporal variation of data distribution.
We test different decay coefficients in Figure~\ref{fig:exp:conf:decay}.
In general, the smaller the coefficient, the easier it is for hot features to drop out as non-hot features.
In experiments, we find that 0.98 is a proper value for decay coefficient in Criteo dataset, while smaller or larger decay coefficient both have poor performance.
When the decay coefficient is too small, hot features cannot stay long in \algo, makes \algo not saturated and hot features mis-classified to non-hot features.
When the decay coefficient is too large, features continuously occupy slots in \algo, even if they are no longer hot features.

\bbb{Other design details.}
We experiment on other design details of \algo.
Currently we maintain only one exclusive embedding table for all fields, instead of maintaining one embedding table per field.
This design makes hot features more flexible, distributed among fields only according to importance scores rather than cardinality.
Figure~\ref{fig:exp:conf:ablation} shows that maintaining only one exclusive embedding table leads to a substantial increase in model AUC.
We also check the effect of using frequency information as importance score, with a worse testing AUC than gradient norm.
Although frequency is a good indicator of feature importance, it has been proved theoretically and experimentally that gradient norm is better.

\subsection{Multi-level Hash Embedding}\label{sec4.4}

In this section, we study the effect of multi-level hash embedding.
The experimental results are shown in Figure~\ref{fig:exp:multi}, where \name-ML means \name combined with multi-level hash embedding.
Under different compression ratios, \name-ML always performs better than \name, achieving $0.08\%$ better testing AUC and reducing $0.25\%$ training loss.
\name-ML performs especially well with smaller compression ratios, causing nearly no degradation at $100\times$ compression ratio.
This is because \name-ML allocates more memory for multi-level hash embedding tables at small compression ratios, making the representation of medium features more precise.

\begin{figure}[htbp]
\centering
\subfigure[AUC v.s. CR.]{
\scalebox{0.25}{
\includegraphics[width=\linewidth]{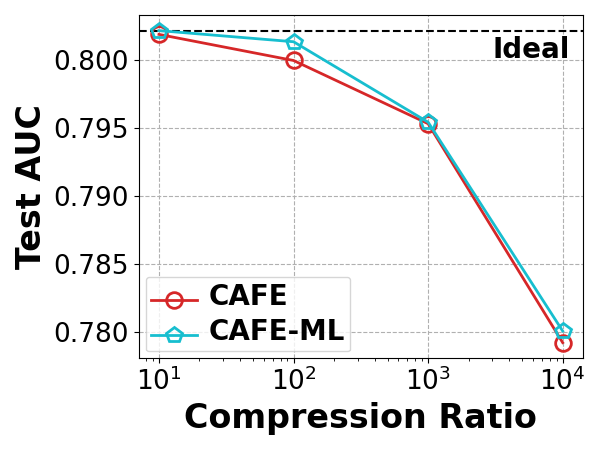}
}
\label{fig:exp:multi:auc}
}
\subfigure[Loss v.s. CR.]{
\scalebox{0.25}{
\includegraphics[width=\linewidth]{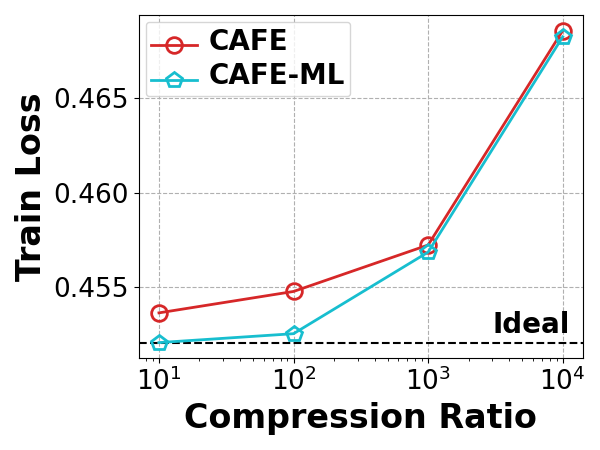}
}
\label{fig:exp:multi:loss}
}\caption{Multi-level hash embedding on Criteo dataset.}
\label{fig:exp:multi}
\end{figure}

    \subsection{Performance on Processed Dataset}\label{sec:expr:manual}

In this section, we construct a new dataset with a more significant shift in data distribution to further validate \name's ability to adapt to changes in data distribution. Keeping the testing data unchanged, we select the training data of days 1,4,7,...,22 from CriteoTB to form CriteoTB-1/3 dataset. As shown in Figure~\ref{fig:intro:kl}, generally the greater the number of days between, the greater the difference between feature distributions. Therefore, CriteoTB-1/3 has a more significant shift in data distribution. The results are shown in Figure~\ref{fig:exp:cropped:criteotb}. Although all methods exhibit slight performance degradation compared to CriteoTB, \name and AdaEmbed can adapt to changing data distributions and achieve relatively good results. Figure~\ref{fig:exp:cropped:iterations} shows that \name and AdaEmbed have almost the same training loss throughout the training process. However, Figure~\ref{fig:exp:cropped:criteotb_auc} and \ref{fig:exp:cropped:criteotb_loss} indicate that \name actually outperforms AdaEmbed with a slight improvement, demonstrating stronger online training capabilities.

\begin{figure}[htbp]
\centering
\subfigure[AUC v.s. CR.]{
\scalebox{0.23}{
\includegraphics[width=\linewidth]{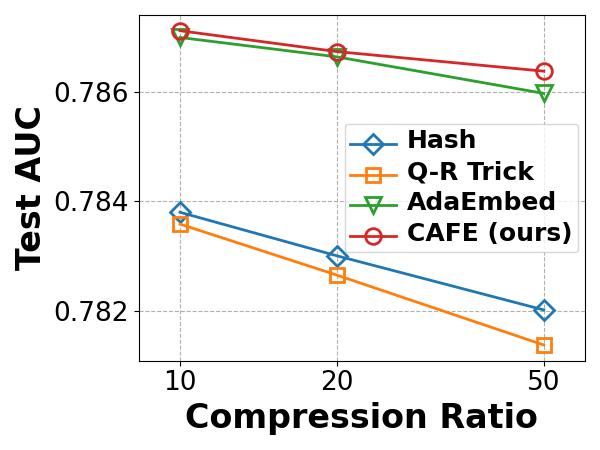}
}
\label{fig:exp:cropped:criteotb_auc}
}
\subfigure[Loss v.s. CR.]{
\scalebox{0.23}{
\includegraphics[width=\linewidth]{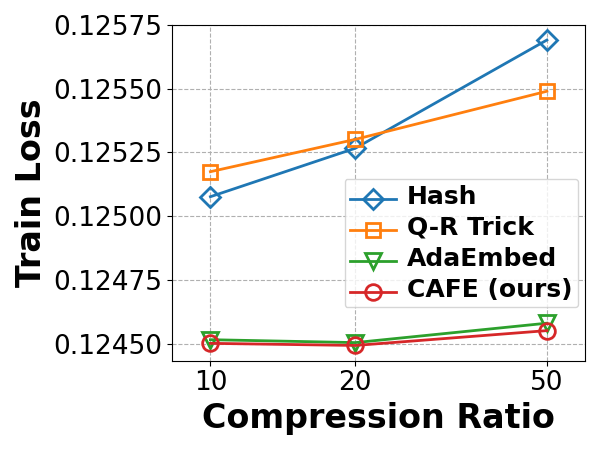}
}
\label{fig:exp:cropped:criteotb_loss}
}
\subfigure[Loss v.s. iterations ($50\times$).]{
\scalebox{0.46}{
\includegraphics[width=\linewidth]{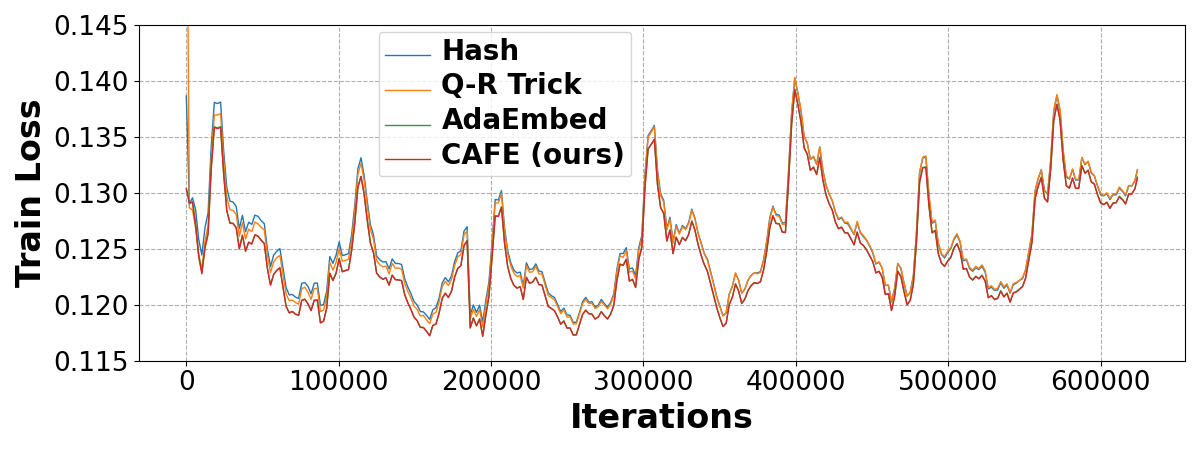}
}
\label{fig:exp:cropped:iterations}
}
\caption{Experiments on CriteoTB-1/3.}
\label{fig:exp:cropped:criteotb}
\end{figure}

\subsection{\algo Performance}\label{subsec:expalgo}

\bbb{Impact of the number of slots per bucket (Figure~\ref{fig:exp:sketch:recall},~\ref{fig:exp:sketch:throughput}):}
We record the recall and the throughput of \algo with different number of slots per bucket.
The experiments use the number of hot features on Criteo dataset ($1000\times$) as $k$.
In Figure~\ref{fig:exp:sketch:recall}, recall generally increases as memory becomes larger.
According to Corollary~\ref{corollary:optimal}, the best number of slots per bucket locates at 11 to 21 given a Zipf distribution of parameter 1.05 to 1.1.
Therefore, $c=8$ and $c=16$ exhibit a better recall than $c=4$ and $c=32$.
The throughput of serialized Insert (write) and Query (read) shown in Figure~\ref{fig:exp:sketch:throughput} is on the order of $1e7$, greater than that of DLRM.
Considering that we can parallelize operations in practice, the sketch time is only a small fraction in training and inference.
Throughput drops as the number of slots increases, because more time is spent doing comparisons within buckets.
Trading-off recall and throughput, we adopt 4 slots per bucket in our implementation, as we find it to be good enough for model quality.

\bbb{Finding real-time top-$k$ features (Figure~\ref{fig:exp:sketch:dynamic100},\ref{fig:exp:sketch:dynamic1000}):}
We conduct experiments to evaluate the performance of HotSketch on finding two types of real-time hot features in online training: the up-to-date top-$k$ features, and the top-$k$ features in previous time window.
These two types of top-$k$ features change with data distribution during the online training process, and thus can effectively reflect HotSketch's capability to adapt dynamic workloads. 
The experiments are conducted on Criteo using 6 days of online training data, with a sliding window size of 0.5 day. 
Figure~\ref{fig:exp:sketch:dynamic100} and \ref{fig:exp:sketch:dynamic1000} show the real-time Recall Rate of HotSketch during online training under different compression ratios. 
HotSketch always achieves $>$90\% Recall Rate on finding these two types of top-$k$ features, meaning that it can well catch up with the changing data distribution.

\begin{figure}[htbp]
\centering
\subfigure[Recall.]{
\scalebox{0.23}{
\includegraphics[width=\linewidth]{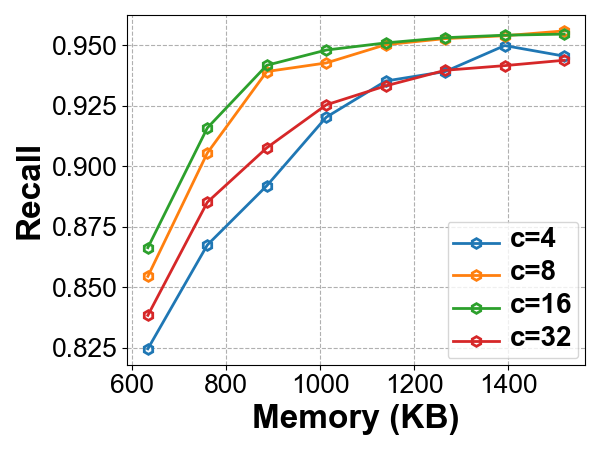}
}
\label{fig:exp:sketch:recall}
}
\subfigure[Throughput.]{
\scalebox{0.23}{
\includegraphics[width=\linewidth]{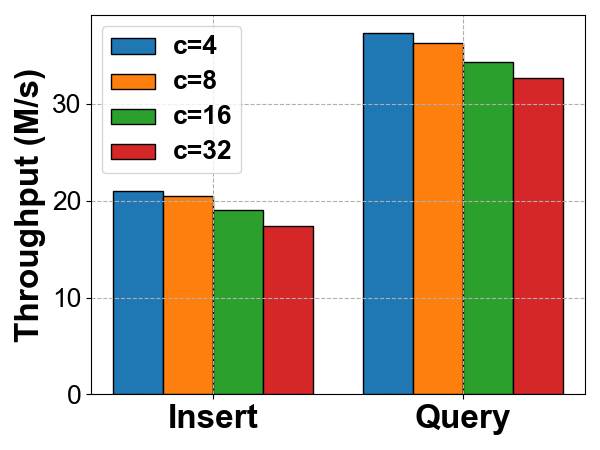}
}
\label{fig:exp:sketch:throughput}
}
\subfigure[Recall v.s. days ($100\times$).]{
\scalebox{0.23}{
\includegraphics[width=\linewidth]{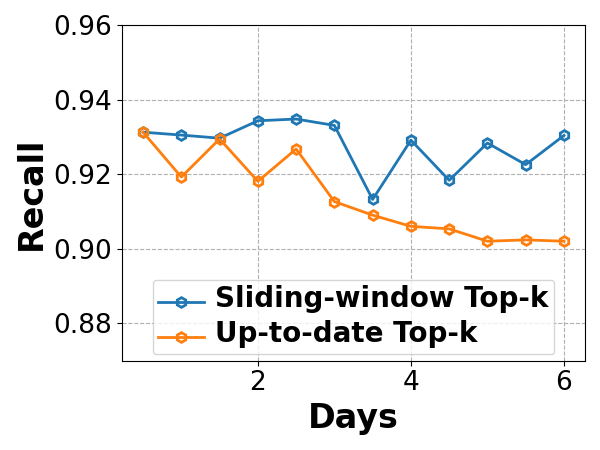}
}
\label{fig:exp:sketch:dynamic100}
}
\subfigure[Recall v.s. days ($1000\times$).]{
\scalebox{0.23}{
\includegraphics[width=\linewidth]{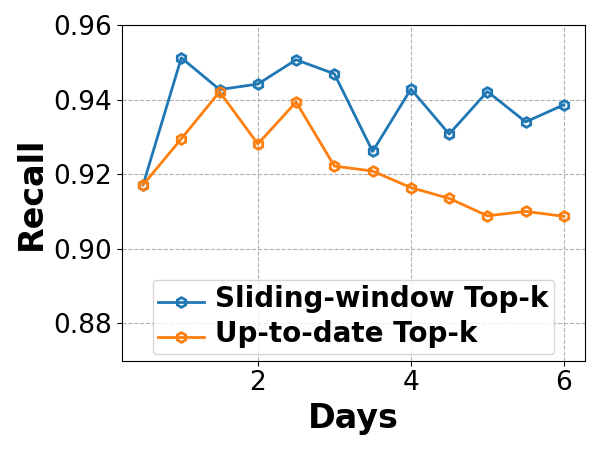}
}
\label{fig:exp:sketch:dynamic1000}
}
\caption{Experiments on \algo.}
\label{fig:exp:sketch}
\end{figure}

\presec
\section{Related Work}
\label{sec:related}
\postsec

\subsection{Embedding Compression}
Numerous compression techniques have been proposed for embedding tables, which can be broadly divided into two categories: row compression and column compression~\cite{DBLP:journals/corr/abs-2311-15578}.
Row compression methods, including hash-based methods, adaptive methods, and \name, reduce the number of rows in embedding tables.
Column compression methods, including quantization, pruning, and dimension reduction, compress each unique feature's representation, thereby reducing the number of columns (or the number of bits) in embedding tables.
Since two categories are primarily orthogonal, methods of different categories can be further combined in DLRMs.

\bbb{Row compression methods:} 
These methods aim to reduce the number of rows in embedding tables.
Initial attempts to accommodate large numbers of embeddings within a limited memory space came from hash-based methods~\cite{DBLP:conf/icml/WeinbergerDLSA09,DBLP:conf/kdd/ShiMNY20,DBLP:conf/cikm/YanWLLLXZ21}, which are widely used in real-world applications.
They use simple hash functions to map categorical features onto a limited number of embeddings, resulting in different features sharing the same embedding vector in the event of hash collisions.
However, hash-based methods do not provide theoretical bounds, which can lead to significant degradation in model quality.
AdaEmbed~\cite{lai2023adaembed} is an adaptive method that identifies and records important features.
It dynamically reallocates embeddings for important features during online training, and achieves good model accuracy over time.
However, its compression ratio is constrained by the storage of importance information, which scales linearly with the number of features.
AdaEmbed's sampling and migration strategy also incurs much latency in online training.

\bbb{Column compression methods:}
Methods of this category aim to compress the representation for each unique feature, thereby reducing the number of columns (or the number of bits) in embedding tables.
They borrow techniques from traditional deep learning compression , such as quantization~\cite{DBLP:conf/sigmod/XuLZSHL021,DBLP:journals/corr/abs-2212-05735}, pruning~\cite{DBLP:conf/wsdm/0002PZKFL21,kong2022autosrh}, and dimension reduction~\cite{DBLP:conf/isit/GinartNMYZ21,DBLP:conf/www/ZhaoLLTGSWGL21,DBLP:conf/cikm/Lyu0ZG0TL22}.
Except for simple quantization and rule-based dimension reduction, most of these methods incorporate learnable structures to implicitly capture the importance of features, achieving similar or even better model accuracy compared to an uncompressed model.
Nevertheless, they are unable to compress the embedding tables to small memory constraints during training.
Specifically, quantization has a fixed compression ratio according to the data type; for example, if INT4 is used for compression, the compression ratio is fixed at $8\times$ compared to FLOAT32.
Generally, pruning and dimension reduction compress the embedding tables only at inference time, requiring additional memory to store extra structures during training.
They are seldom used in industry, as the memory bottleneck during training is more severe due to activations and optimizer states.
Most of these methods can only support offline training because they require collected data for multi-stage training, including pre-training, finetuning, and re-training.

\subsection{Sketching Algorithm}

Sketch is an excellent probabilistic data structure that can approximately record the statistics of data streams by maintaining a summary. 
Thanks to their small memory overhead and fast processing speed, sketches are widely applied in the realm of streaming data mining~\cite{cmsketch}, database~\cite{izenov2021compass,shi2021time,chiosa2021skt,DBLP:journals/pacmmod/LiuZZZ0XWL023}, and network measurement and management~\cite{yang2018elastic,zhang2021cocosketch} to perform various tasks, such as frequency estimation~\cite{cmsketch,cusketch,cmmsketch}, finding top-$k$ frequent items~\cite{spacesaving,wavingsketch,heavyguardian,mandal2018topkapi}, and mining special patterns in streaming data~\cite{liu2023hypercalm}.
Existing sketches can be classified into two categories: counter-based sketches and KV-based sketches.

\bbb{Counter-based sketches:}
Typical counter-based sketches include CM~\cite{cmsketch}, CU~\cite{cusketch}, Count~\cite{count}, ASketch~\cite{asketch}, and more~\cite{cmlsketch,fan2008efficient,csmsketch,cmmsketch}. 
The data structures of these sketches usually consist of multiple arrays, each containing many counters.
Each array is associated with one hash function that maps items into a specific counter in it. 
For example, the most popular CM sketch~\cite{cmsketch} comprises $d$ counter arrays $C_1, \cdots, C_d$. 
For each incoming item $e$, it is hashed into $d$ counters in the CM sketch $C_1[h_1(e)], \cdots, C_d[h_d(e)]$ with each of the $d$ counters incremented by one. 
To query item $e$, CM sketch returns the minimum counter among $C_1[h_1(e)], \cdots, C_d[h_d(e)]$. 
However, the CM sketch has overestimated errors due to hash collisions. 
Other sketches propose various strategies to reduce this error. 
For instance, CU sketch~\cite{cusketch} only increments the minimum counter among $C_1[h_1(e)], \cdots, C_d[h_d(e)]$, and Count sketch~\cite{count} adds $s(e) \in \{+1, -1\}$ to each mapped counter to achieve unbiased estimation. 
Despite these improvements, existing counter-based sketches are not memory efficient for finding top-$k$ items.
They do not distinguish between frequent and infrequent items, where infrequent items are useless for reporting top-$k$ items, and recording infrequent items only increases the error of frequent items.
Moreover, they need multiple memory accesses per insertion, resulting in unsatisfactory insertion speed.

\bbb{KV-based sketches:}
Common key-value-based sketches include Space-Saving~\cite{spacesaving}, Unbaised Space-Saving~\cite{DBLP:conf/sigmod/Ting18}, Lossy Counting~\cite{proht}, HeavyGuardian~\cite{heavyguardian}, and more~\cite{wavingsketch,yang2018elastic,zhang2021cocosketch}.
These sketches maintain the KV pairs of frequent items in their data structures. 
For instance, Space-Saving~\cite{spacesaving} and Unbiased Space-Saving~\cite{DBLP:conf/sigmod/Ting18} use a data structure called Stream-Summary to record frequent items, which is essentially a doubly-linked list of fixed size, indexed by a hash table.
When Stream-Summary is full and an unrecorded item arrives, Space-Saving replaces the least frequent item with the incoming one.
Based on Space-Saving, Unbiased Space-Saving~\cite{DBLP:conf/sigmod/Ting18} replaces the least frequent item with a certain probability, so as to achieve unbiased estimation. 
Unfortunately, Space-Saving and Unbiased Space-Saving are not memory and time efficient because of the extra hash table and complex pointer operations. 
Another type of KV-based sketches, such as HeavyGuardian~\cite{heavyguardian} and WavingSketch~\cite{wavingsketch}, uses a bucket array data structure, where each bucket stores multiple KV pairs. 
These sketches provide satisfactory accuracy for reporting top-$k$ items and only require one memory access per insertion, ensuring fast insertion speed. 

\presec
\section{Conclusion}\label{sec:conclusion}
\postsec

In this paper, we propose \name, a compact, adaptive, and fast embedding compression method that fulfills three essential design requirements: memory efficiency, low latency, and adaptability to dynamic data distribution.
We introduce a light-weight sketch structure, \algo, to identify and record the importance scores of features.
It incurs negligible time overhead, and its memory consumption is significantly lower than the original embedding tables.
By assigning exclusive embeddings to a small set of important features and shared embeddings to other less important features, we achieve superior model quality within a limited memory constraint.
To adapt to dynamic data distribution during online training, we incorporate an embedding migration strategy based on \algo.
We further optimize \name with multi-level hash embedding, creating finer-grained importance groups.
Experimental results demonstrate that \name outperforms existing methods, with $3.92\%$, $3.68\%$ higher testing AUC  and $4.61\%$, $3.24\%$ lower training loss at $10000\times$ compression ratio on Criteo Kaggle and CriteoTB datasets, exhibiting superior performance in both offline training and online training.
The source codes of \name are available at GitHub~\cite{cafecode}.

\begin{acks}
This work is supported by National Key R\&D Program of China (2022ZD0116315), National Natural Science Foundation of China (U22B2037, U23B2048, 62372009), PKU-Tencent joint research Lab.
\end{acks}

\bibliographystyle{ACM-Reference-Format}
\bibliography{reference}

\end{document}